\documentclass{article}

\usepackage[preprint]{neurips_2025}

\usepackage[utf8]{inputenc} %
\usepackage[T1]{fontenc}    %
\usepackage{hyperref}       %
\usepackage{url}            %
\usepackage{booktabs}       %
\usepackage{amsfonts}       %
\usepackage{nicefrac}       %
\usepackage{microtype}      %
\usepackage{xcolor}         %
\usepackage{enumitem}

\usepackage{hyperref}

\usepackage{amsmath}
\usepackage{amssymb}
\usepackage{mathtools}
\usepackage{amsthm}

\usepackage[capitalize,noabbrev]{cleveref}

\theoremstyle{plain}
\newtheorem{theorem}{Theorem}[section]
\newtheorem{proposition}[theorem]{Proposition}
\newtheorem{lemma}[theorem]{Lemma}
\newtheorem{corollary}[theorem]{Corollary}
\theoremstyle{definition}

\theoremstyle{remark}
\newtheorem{remark}[theorem]{Remark}
\usepackage{soul}

\usepackage[textsize=tiny]{todonotes}
\usepackage{algorithm}
\usepackage{algorithmic}
\usepackage{soul}
\usepackage{wrapfig}
\usepackage{xspace}

\title{Bounded Ratio Reinforcement Learning}

\author{
  Yunke Ao \thanks{Correspondence to: \texttt{yunke.ao@ai.ethz.ch}}\\
  ETH Zurich\\
  \And
  Le Chen \thanks{Equal second author contribution}\\
  MPI for Intelligent Systems \\
  \And
  Bruce D. Lee \footnotemark[2] \\
  ETH Zurich\\
  \And
  Assefa S. Wahd \thanks{Equal third author contribution}\\
  University of Alberta\\
  \And
  Aline Czarnobai \footnotemark[3] \\
  Dartmouth College \\
  \And 
  Philipp Fürnstahl \\
  Balgrist University Hospital \\
  \And
  Bernhard Schölkopf \\
  MPI for Intelligent Systems \\
  \And 
  Andreas Krause \\
  ETH Zurich \\
}

\begin{document}

\newcommand{\bruce}[1]{{\color{blue}{BL: #1}}}
\newcommand{\yunke}[1]{{\color{red}{YA: #1}}}
\maketitle

\begin{abstract}
    \looseness -1 Proximal Policy Optimization (PPO) has become the predominant algorithm for on-policy reinforcement learning due to its scalability and empirical robustness across domains. However, there is a significant disconnect between the underlying foundations of trust region methods and the heuristic clipped objective used in PPO. 
    In this paper, we bridge this gap by introducing the {\em Bounded Ratio Reinforcement Learning (BRRL)} framework.
    We formulate a novel regularized and constrained policy optimization problem and derive its analytical optimal solution. We prove that this solution ensures monotonic performance improvement. To handle parameterized policy classes, we develop a policy optimization algorithm called {\em Bounded Policy Optimization (BPO)} that minimizes an advantage-weighted divergence between the policy and the analytic optimal solution from BRRL.  We further establish a lower bound on the expected performance of the resulting policy in terms of the BPO loss function. Notably, our framework also provides a new theoretical lens to interpret the success of the PPO loss, and connects trust region policy optimization and the Cross-Entropy Method (CEM). We additionally extend BPO to {\em Group-relative BPO (GBPO)} for LLM fine-tuning. Empirical evaluations of BPO across MuJoCo, Atari, and complex IsaacLab environments (e.g., Humanoid locomotion), and of GBPO for LLM fine-tuning tasks, demonstrate that BPO and GBPO generally match or outperform PPO and GRPO in stability and final performance.
\end{abstract}

\begin{figure*}[!hbtp] 
    
    \centering %
    
    \includegraphics[width=1.0\textwidth]{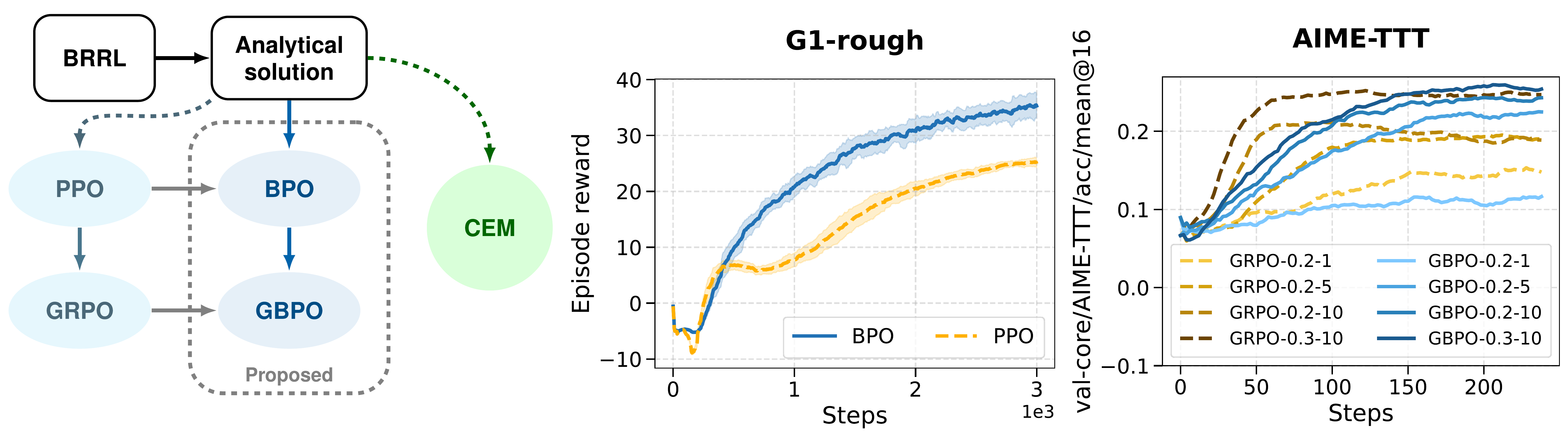} 
    
    \caption{
    The Bounded Ratio Reinforcement Learning (BRRL) framework introduces the surrogate policy optimization problem under bounded ratio constraints. Its analytical solution closely relates to the PPO objective function, cross-entropy methods (CEM), and suggests a theoretically grounded policy optimization algorithm with minor changes to PPO: Bounded Policy Optimization (BPO). We observe marked improvements of BPO and its variant Group-relative BPO (GBPO) in the performance and stability for humanoid locomotion and LLM fine-tuning (mathematical reasoning).
    }
    \label{fig:title}
\end{figure*}

\section{Introduction}

Deep reinforcement learning (DRL) has achieved breakthroughs across diverse domains
~\cite{silver2017mastering,lee2020learning,ouyang2022training,radosavovic2024real}. 
Among DRL methods, Proximal Policy Optimization (PPO)~\cite{schulman2017proximal} remains one of the most widely adopted algorithms.
The core design of PPO is motivated by Trust Region Policy Optimization (TRPO,~\cite{schulman2015trust}), which constrains policy updates within a ``trust region'' to ensure stable iterations. 
By utilizing a first-order approximation of the TRPO objective, PPO achieves the scalability necessary for training modern large-scale models. 
As a result, PPO and its variant GRPO are now widely applied to tasks ranging from robotics to large language model (LLM) fine-tuning~\cite{miki2022learning,andrychowicz2020learning,shao2024deepseekmath}.

\looseness -1 Despite its empirical success, PPO remains largely heuristic: its clipped objective is not directly derived from the trust-region formulation it was intended to approximate. Instead, the design of the PPO objective was primarily driven by experimentation~\cite{schulman2017proximal,engstrom2020implementation}. 
Furthermore, most existing theoretical analyses of PPO's performance improvement rely on the original TRPO or policy gradient formulation~\cite{schulman2015trust,liu2019neural,doering2026approximate}, none of which fully capture the nuances of the first-order loss used in practice. 

Numerous variants have been recently proposed to improve PPO. 
Some works focus primarily on algorithm design and report empirical performance gains without formal theoretical contributions~\cite{cobbe2021phasic,ye2020mastering,tan2024beyond,fakoor2020p3o,kobayashi2021proximal}.
Other works extend PPO to specific domains (e.g., safe RL, non-stationary RL) without modifying the core PPO loss function~\cite{akgul2025overcoming,milosevic2025central}.
There are also PPO variants aiming at improving the PPO loss from a theoretical lens~\cite{xie2024simple,wang2020truly,wang2019trust,qi2026rethinking}.
However, similar to PPO, they also utilize TRPO theory without introducing novel theoretical frameworks or establishing superior performance guarantees.
Consequently, there remains a substantial gap between the theoretical foundations and the practical policy optimization algorithms.

To address this gap, we introduce the 
 bounded ratio reinforcement learning (BRRL) framework. Instead of constraining policy updates through KL divergence~\cite{kullback1951information} bounds as in TRPO, BRRL imposes bounded ratio constraints on the policy likelihood ratios.
 This formulation admits an analytic optimal policy, which reveals a simple structure for policy updates. %
 We establish the following contributions using the BRRL framework:
 \begin{itemize}[noitemsep, nolistsep, leftmargin=*]
     \item We derive the optimal solution of BRRL and prove its monotonic performance improvement guarantees. We also demonstrate that optimizing the PPO loss approximately pushes the policy towards this analytic optimal solution. 
     \item We establish a connection between BRRL and the Cross-Entropy Method (CEM).
     \item We propose Bounded Policy Optimization (BPO), which optimizes an advantage-weighted divergence from the BRRL solution. We also extend BPO to Group-Relative BPO (GBPO), mirroring the extension from PPO to GRPO.
     \item We provide a performance improvement guarantee for the policy attained by BPO in terms of the loss that BPO optimizes.
     \item We demonstrate strong empirical performance of BPO on MuJoCo, Atari, IsaacLab locomotion tasks, and of GBPO for LLM fine-tuning.
 \end{itemize}
 Overall, BRRL provides a principled perspective on PPO-style algorithms, suggesting that their empirical success arises from approximating an analytically optimal bounded-ratio update. By more directly approximating this analytically optimal bounded-ratio update, BPO achieves improved empirical performance (\Cref{fig:title}).

\section{Notation}
\label{sec:mdp}

\textbf{Markov Decision Process (MDP):} We consider an infinite-horizon MDP defined by the tuple $(\mathcal{S},\mathcal{A},\mathcal{P},r,d, \gamma)$, where $\mathcal{S}$ is the state space, $\mathcal{A}$ is the action space, $\mathcal{P}:\mathcal{S}\times\mathcal{A}\times \mathcal{S}\rightarrow \mathbb{R}$ is the transition model, $r:\mathcal{S}\times\mathcal{A}\times \mathcal{S}\rightarrow \mathbb{R}$ is the reward function, $d_0: \mathcal{S}\rightarrow\mathbb{R}$ is the initial state distribution,
 and $\gamma\in(0, 1)$ is the discount factor. 
Let $\pi:\mathcal{\mathcal{S}}\times\mathcal{A}\rightarrow\mathbb{R}$ denote the stochastic policy. 
We denote $r_t := r(s_t, a_t, s_{t+1})$.
The goal of the MDP is to solve the optimization problem
\begin{equation}
\begin{split}
    \max_{\pi\in \Pi} \eta(\pi):=\mathbb{E}_{s_{0:\infty},a_{0:\infty}}\left[\sum_{t=0}^{\infty}\gamma^t r_t\right], \,\,s_0\sim d_0,\,a_t\sim \pi(a_t|s_t), s_{t+1}\sim p(s_{t+1}|s_t,a_t),
\end{split}\label{eq:org_opt}
\end{equation}
which maximizes the expected discounted return under policy $\pi$ within the policy class $\Pi$.
Let us denote $d_\pi(s):=\sum_{t=0}^\infty \gamma^t P(s_t=s)$ as an unnormalized state visitation distribution~\cite{schulman2015trust}. 
Then the objective in~\eqref{eq:org_opt} can be rewritten as $\eta(\pi)=\mathbb{E}_{s\sim d_\pi, a\sim \pi(\cdot|s),s'\sim P(\cdot|s, a)}[r(s, a, s')]$.

\looseness -1 \textbf{Value function and advantage:} We define the value function of a state $s$ given policy $\pi$ as $V_\pi(s):=\mathbb{E}_{s_0,a_0,...|s_0=s}[\sum_{t=0}^{\infty}\gamma^t r_t]$, and the Q-function $Q_\pi(s, a):=\mathbb{E}_{s_0,a_0,...|s_0=s, a_0=a}[\sum_{t=0}^{\infty}\gamma^t r_t]$, where the actions (excluding the conditioned $a_0$ in the Q-function) are sampled from the policy $\pi$. 
The advantage function is defined as the difference between them $A_\pi(s,a):=Q_\pi(s,a)-V_\pi(s)$.
As shown in~\cite{schulman2015trust}, the expected return of the new policy $\pi$ in~\eqref{eq:org_opt} with regard to an old policy $\pi_0$ can be derived as 
\begin{equation}
    \eta(\pi)=\eta(\pi_0) + \mathbb{E}_{s\sim d_\pi, a\sim \pi(\cdot|s)}[A_{\pi_0}(s, a)]. \label{eq:eta_adv}
\end{equation}
where the advantage is evaluated under $\pi_0$, and the expectation is taken over $d_\pi$ and $\pi$.

\textbf{Surrogate objectives:} We denote the surrogate objective optimized by TRPO~\cite{schulman2015trust} as
\begin{equation}
    L_{\pi_0}(\pi):=\eta(\pi_0) + \mathbb{E}_{s\sim d_{\pi_0}, a\sim \pi(\cdot|s)}[A_{\pi_0}(s, a)] = \eta(\pi_0) + \mathbb{E}_{s\sim d_{\pi_0}, a\sim \pi_0(\cdot|s)}[\rho A_{\pi_0}(s, a)] ,\label{eq:l}
\end{equation}
where $\rho=\rho(a| s):=\pi(a|s)/\pi_0(a|s)$ are the importance weights.
In contrast to~\eqref{eq:eta_adv}, $L_{\pi_0}(\pi)$ takes the expectation over $d_{\pi_0}$ instead of $d_\pi$.
In TRPO, $\pi$ is updated to optimize $L_{\pi_0}(\pi)$ with constrained KL-divergence from $\pi_0$.

\section{Overview of Contributions}
\begin{figure*}[t] 
    
    \centering %
    
    \includegraphics[width=0.9\textwidth]{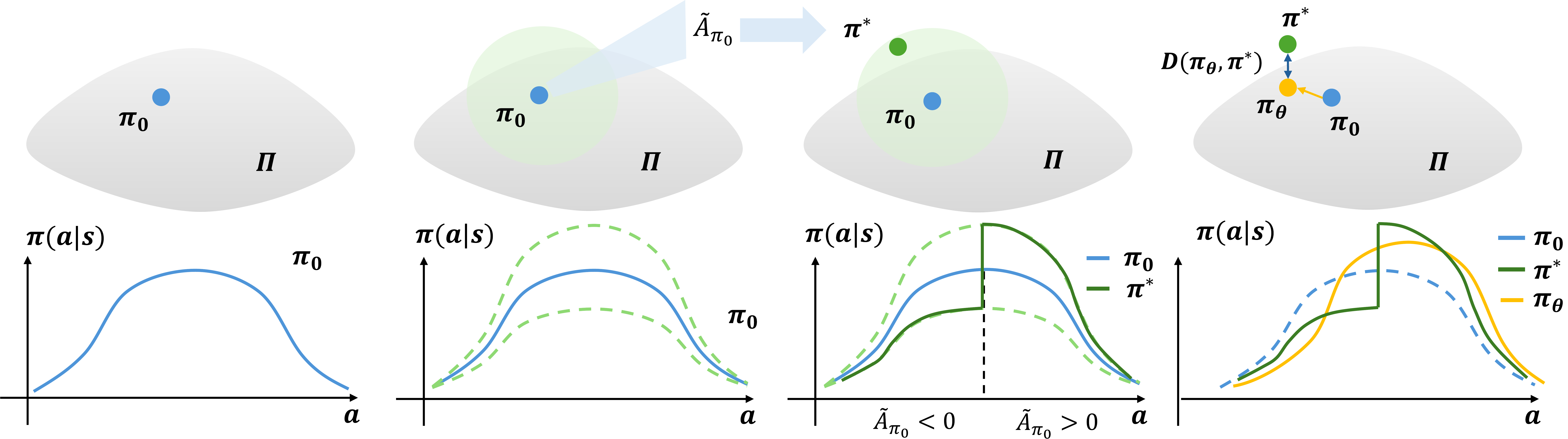} 
    
    \caption{Illustration of Bounded Ratio RL (BRRL). (Left) Old policy $\pi_0$ within the parameterized policy class $\Pi$. 
    (Middle Left) Construction of a trust region (light green) defined by bounded ratio constraints from Problem~\eqref{prob:bounded_ratio} or~\eqref{eq:conservative}. 
    (Middle Right) Estimation of the analytic optimal policy within the trust region (dark green, can be outside $\Pi$) using (soft-)median-advantages $\Tilde{A}_{\pi_0}$ (Theorem~\ref{thr:conservative}).
    In general, actions with positive (resp.~negative) advantages $\Tilde{A}_{\pi_0}>0$ (resp. $\Tilde{A}_{\pi_0}<0$) yield optimal ratios greater (resp. smaller) than 1.
    (Right) The updated policy within $\Pi$ (yellow) is obtained by minimizing a divergence from the estimated optimal policy.
    }
    \label{fig:main}
\end{figure*}

\label{sec:overview}
In this section, we present an overview of the contributions within this work, as shown in Figure~\ref{fig:main}. %

\textbf{Bounded ratio RL framework: }
\looseness -1 We consider a policy optimization problem with bounded ratio trust region constraints from an old policy $\pi_0$, instead of the KL-divergence constraint of TRPO, as shown in Figure~\ref{fig:main} (Middle Left).
Specifically, with $L_{\pi_0}(\pi)$ defined in~\eqref{eq:l}, the problem is expressed as
\begin{equation}
    \max_{\pi} L_{\pi_0}(\pi),\quad 
    \text{s.t.}\,\, 1-\epsilon \leq \frac{\pi(a|s)}{\pi_0(a|s)}\leq 1+\epsilon, \,\,\forall\,s,a.,~\label{prob:bounded_ratio}
\end{equation}
which also has an implicit normalization constraint $\sum_{a} \pi(a|s)=1,\forall s$.
Notably, this problem has an \emph{analytical} optimal solution $\pi^*$, which in many cases (as detailed in Remark~\ref{rmk:unreg}) can be derived as 
\begin{equation}
    \pi^*(a|s) = [1+\epsilon\cdot\text{sign}(\tilde{A}_{\pi_0})]\cdot\pi_0(a|s),
    \label{eq:simple_solution}
\end{equation}
\looseness -1 where $\tilde{A}_{\pi_0} := Q_{\pi_0}(s,a) - \mu_{\pi_0}(s)$ is the median advantage. In particular, $\mu_{\pi_0}(s)$ denotes the median of $Q_{\pi_0}(s,a)$ over $\pi_0$, such that for any $s \in \mathcal S$, $\tilde A_{\pi_0}$ satisfies $\mathbb{E}_{a\sim\pi_0(\cdot|s)}[\text{sign}(\tilde{A}_{\pi_0})]=0$. 
As shown in Figure~\ref{fig:main} (Middle Right) and Figure~\ref{fig:bpo_ppo} (c), this optimal solution can be explained as: if $Q_{\pi_0}(s, a)$ is higher than the threshold $\mu_{\pi_0}(s)$, then take the highest probability within the constraint $\pi^*(a|s)=(1+\epsilon)\pi_0(a|s)$; otherwise, let $\pi^*(a|s)=(1-\epsilon)\pi_0(a|s)$.
Threshold $\mu_{\pi_0}(s)$ is selected as the median, s.t. $\pi^*$ is a normalized probability distribution ($\sum_a\pi^*(a|s)=1$).
A formal theorem on the optimal solution for general cases is provided in Theorem~\ref{thr:conservative}.
Note that Theorem~\ref{thr:conservative} can also be extended to problems with asymmetric bounded ratio constraints ($c_l\leq \pi(a|s)/\pi_0(a|s) \leq c_h$). This asymmetric solution is used to draw a connection to the cross-entropy method (CEM,~\cite{rubinstein1999cross}) in Section~\ref{sec:cem}. %

\begin{figure*}[t] 
    
    \centering %
    
    \includegraphics[width=1.0\textwidth]{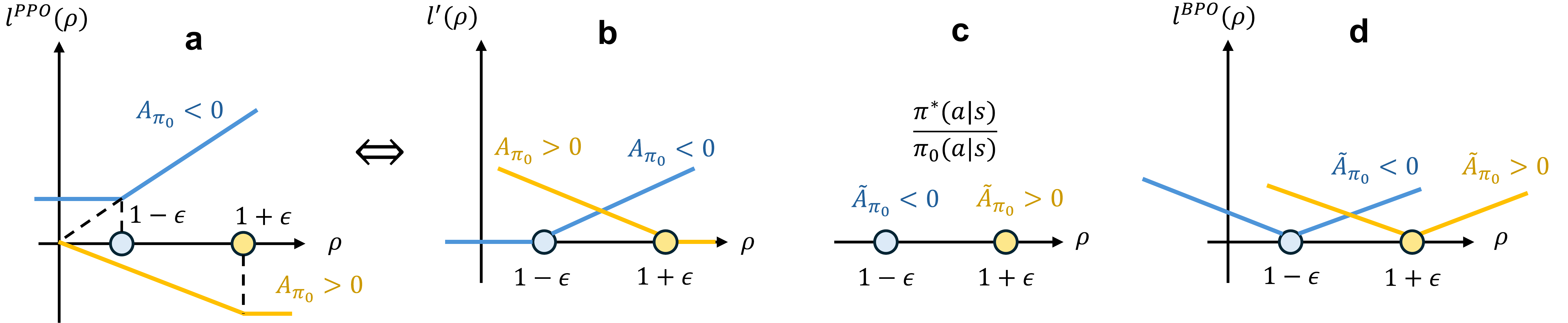} 
    
    \caption{
    Loss functions of PPO and bounded-ratio RL.
    Curves for $\Tilde{A}_{\pi_0}>0$ and $\Tilde{A}_{\pi_0}<0$ are shown in yellow and blue, respectively.
    (a) Original PPO loss function. (b) Equivalent loss function of PPO as introduced in Equation~\eqref{eq:ppo_sim_l}. (c) Optimal ratios for the optimization problem with bounded ratio constraints in~\eqref{prob:bounded_ratio}. (d) Advantage-weighted TV loss function in BPO, defined in~\eqref{eq:simple_loss}.}
    \label{fig:bpo_ppo}
\end{figure*}

\textbf{Monotonic performance guarantee:} For the cases where optimal policy $\pi^*$ from~\eqref{eq:simple_solution} is realizable, it can be shown to have improved performance over $\pi_0$
\begin{equation}
    \eta(\pi^*) = \eta(\pi_0) + \epsilon\,\mathbb{E}_{s\sim d_{\pi^*},a\sim \pi_0}[\text{sign}(\tilde{A}_{\pi_0})\tilde{A}_{\pi_0}]:=\eta(\pi_0) + \epsilon B, \label{eq:simple_thr}
\end{equation}
\looseness -1 where the second term is non-negative and is positive whenever $\pi_0$ induces non-zero median advantage.
For a fixed $\pi_0$, we denote this constant improvement term as $\epsilon B$.
A performance bound for general cases is provided in Theorem~\ref{thr:soft_guarantee}.
Though $\pi^*$ is simple to express and provides improvement guarantees, it may not lie in the admissible policy class $\Pi$ (Figure~\ref{fig:main} Middle right).
This motivates the design of policy optimization algorithms 
to minimize divergence between the policy $\pi\in\Pi$ and $\pi^*$.

\textbf{Revisiting the PPO loss function: }
We observe that the PPO loss function approximately drives the policy towards $\pi^*$ in~\eqref{eq:simple_solution}.
Specifically, as shown in Figure~\ref{fig:bpo_ppo} (a-b), %
optimizing the PPO objective~\cite{schulman2017proximal} is equivalent to minimizing the expectation of the following loss function evaluated at $\rho = \pi(a|s)/\pi_0(a|s)$
\begin{equation}
\begin{split}
        &l'(\rho ):=\begin{cases}|A_{\pi_0}|\cdot |\rho - (1+\epsilon\cdot\text{sign}(A_{\pi_0}) )|, & [\rho - (1+\epsilon\cdot\text{sign}(A_{\pi_0}))]\cdot A_{\pi_0} \leq 0, \\[2pt]
0, & \text{Otherwise}. 
\end{cases}
\end{split}\label{eq:ppo_sim_l}
\end{equation}
A formal theorem on this equivalence with step-by-step proof is detailed in Section~\ref{sec:revisitPPO} and Appendix~\ref{app:proof_ppo_equ}. At the beginning of the iteration, the ratio always starts from 1, and the PPO loss minimizes an \emph{advantage-weighted absolute error} between the ratio $\rho$ and the target $1 + \epsilon\text{sign}(A_{\pi_0})$, then it applies zero-gradient after reaching the target.
Note that this target ratio closely matches the solution in~\eqref{eq:simple_solution}, except that PPO uses the mean advantage $A_{\pi_0}$, and the BRRL solution is expressed in terms of the median advantage $\tilde{A}_{\pi_0}$.

\textbf{Bounded Policy Optimization (BPO): } 
Building on the solution in~\eqref{eq:simple_solution}, we introduce a natural PPO variant with the loss function $l^{BPO}$ to directly minimize the \emph{advantage weighted total variation} from the optimal solution.
For the solution in~\eqref{eq:simple_solution}, the loss $l^{BPO}$ evaluated under $\rho = \pi(a|s)/\pi_0(a|s)$ is
\begin{equation}
    l^{BPO}(\rho ):=|A_{\pi_0}|\cdot \left|\rho - \frac{\pi^*(a|s)}{\pi_0(a|s)}\right|=|A_{\pi_0}|\cdot |\rho - (1+ \epsilon\cdot\text{sign}(\tilde{A}_{\pi_0}))|.\label{eq:simple_loss}
\end{equation}
\looseness -1 The loss is illustrated in Figure~\ref{fig:bpo_ppo} (d).
Compared with the PPO loss in~\eqref {eq:ppo_sim_l}, this loss function $l^{BPO}$ only differs in two ways:  (1) a symmetric slope also for $[\rho - (1+\epsilon\cdot\text{sign}(A_{\pi_0}))]\cdot A_{\pi_0} > 0$ and (2) using $\tilde{A}_{\pi_0}$ instead of $A_{\pi_0}$.
In practice, this also requires learning an additional median value function alongside the mean value function, though the median can be approximated by the mean to reduce computational overhead. %
Notably, with this refined loss function~\eqref{eq:simple_loss}, BPO has both \emph{theoretical performance guarantees} (discussed below) and strong empirical performance, as demonstrated in Section~\ref{sec:exp}.
The same loss function can also be adapted for LLM fine-tuning, analogous to how PPO was adapted to GRPO (Section~\ref{sec:gbpo}).

\textbf{BPO performance guarantees: }
Assuming the optimal solution in~\eqref{eq:simple_solution} is valid, we can express the stepwise improvement in terms of the achieved loss~\eqref{eq:simple_loss}. %
Specifically, we show that
\begin{align*}
    \eta(\pi) \geq \eta(\pi_0) + \epsilon B - \mathbb{E}_{
        s\sim d_{\pi_0},
        a\sim \pi_0}\left[l^{BPO}\left(\frac{\pi(a|s)}{\pi_0(a|s)}\right)\right] - \delta(\pi,\pi^*),
\end{align*}
where $B$ is defined in~\eqref{eq:simple_thr}. Here, $\delta(\pi,\pi^*)$ is an error term that is related to $l^{BPO}(\frac{\pi(a|s)}{\pi_0(a|s)})$ and reduces to $0$ if we have perfect policy approximation $\pi=\pi^*$.
This theoretical result directly implies that, if our loss function $l^{BPO}$ is sufficiently minimized over states and actions sampled from $\pi_0$, and if the policy approximation error is small, we can obtain monotonic performance improvement.
The formal result is detailed in Corollary~\ref{coro:loss_tv}.

\section{Method}

 We now proceed to present the aforementioned framework of BRRL and its extensions.

\subsection{Bounded Ratio RL Framework}
\label{sec:optimal solution}
Intuitively, for an MDP with finite state and action spaces, Problem~(\ref{prob:bounded_ratio}) is a linear programming problem.
Specifically, for a fixed state $s$, the optimization variable $\pi(a|s)$ is a finite-dimensional vector.
Consequently, the objective function and constraints in Problem~\eqref{prob:bounded_ratio} are linear in $\pi(a|s)$.
However, for general state and action spaces, the optimal solution of this linear programming problem is difficult to specify analytically.
Nevertheless, an additional \emph{regularizer} allows for the derivation of the general analytical solution.
Namely, we consider the following regularized constrained optimization problem:
\begin{equation}
\begin{split}
    &\max_{\pi}  \, L_{\pi_0}(\pi) - \lambda \mathbb{E}_{s\sim d_{\pi_0},a\sim \pi_0}\left[H\left(\frac{\pi(a|s)}{\pi_0(a|s)}\right)\right], \\
    \text{where} \quad& H(\rho):=(\rho - 1 + \epsilon) \log(\rho - 1 + \epsilon)  + (1 + \epsilon - \rho)\log(1 + \epsilon - \rho). \label{eq:conservative}
\end{split}
\end{equation}
\looseness -1 Here the regularizer $H(\rho)\in[2\epsilon\log\epsilon, 2\epsilon\log(2\epsilon)]$  decreases as $\rho \rightarrow 1$ and increases as $\rho \rightarrow 1\pm\epsilon$.
Moreover, its gradient becomes unbounded near the boundaries $1\pm\epsilon$, so $H$ provides log barriers for the original bounded ratio constraints $1-\epsilon < \frac{\pi(a|s)}{\pi_0(a|s)} < 1+\epsilon$.
The regularizer is weighted by $\lambda$.
According to Fermi-Dirac statistics~\cite{landau1980statistical}, Problem~\eqref{eq:conservative} has a closed-form solution, detailed in the following theorem.
\begin{theorem}[Optimal solution]
    The optimal policy $\pi^*$ of the problem described in~\eqref{eq:conservative} satisfies:
    \begin{align*}
       \pi^*(a|s) =\left(1+\epsilon \tanh \left(\frac{\tilde{A}_{\pi_0}}{2\lambda}\right)\right)\pi_0(a|s),\,\, \tilde{A}_{\pi_0}:=Q_{\pi_0}(s, a) - \mu_{\pi_0}(s),
        \end{align*}
        where $\mu_{\pi_0}(s)$ is called the soft-median of $Q_{\pi_0}(s,a)$ that satisfies
        \begin{align*}
     \mathbb{E}_{a\sim\pi_0(\cdot|s)}\left[\tanh \left(\frac{\tilde{A}_{\pi_0}}{2\lambda}\right)\right] = 0 \,\, \Leftrightarrow \,\, \mu_{\pi_0}(s)=\arg\min_{\mu(s)} \mathbb{E}_{a\sim\pi_0(\cdot|s)}\left[ g\left(\frac{Q_{\pi_0}(s, a) - \mu(s)}{\lambda}\right)\right],
    \end{align*} where $g:\mathbb{R}\rightarrow \mathbb{R}_{\geq 0},g(x)=\ln(e^{-\frac{x}{2}}+e^{\frac{x}{2}})$ is a soft absolute function. \label{thr:conservative}
\end{theorem}

The detailed proof of Theorem~\ref{thr:conservative} is provided in Appendix~\ref{sec:proof_policy}.
Intuitively, the optimal solution assigns a higher ratio to actions with a higher advantage while keeping the ratio between $[1-\epsilon, 1+\epsilon]$.

We can obtain a monotonic performance guarantee for the optimal solution.
\begin{theorem}[Monotonic performance guarantee]
    The optimal policy in Theorem~\ref{thr:conservative} satisfies
    \begin{align*}
        &\eta(\pi^*) = \eta(\pi_0) + \epsilon \mathbb{E}_{s\sim d_{\pi^*}, a\sim \pi_0(\cdot|s)}\left[\tanh\left(\frac{\tilde{A}_{\pi_0}}{2\lambda}\right)\tilde{A}_{\pi_0}\right]=:\eta(\pi_0) + \epsilon B,
    \end{align*}
    where $\tilde{A}_{\pi_0}$ abbreviates $\tilde{A}_{\pi_0}(s,a)$, $B$ is a non-negative constant given fixed $\pi_0$.
     \label{thr:soft_guarantee}
\end{theorem}

The proof of Theorem~\ref{thr:soft_guarantee} is detailed in Appendix~\ref{sec:proof_mono}.
Note that the term $\tanh(\frac{\tilde{A}_{\pi_0}}{2\lambda})\tilde{A}_{\pi_0}$ is always non-negative since the signs of $\tanh(\frac{\tilde{A}_{\pi_0}}{2\lambda})$ and $\tilde{A}_{\pi_0}$ are always the same. 
Therefore, our optimal policy $\pi^*$ guarantees monotonic improvement with an analytical improvement bound, in contrast to the guarantee for TRPO in Theorem 1 of \cite{schulman2015trust}.
However, the policy $\pi^*$ may not be a member of the class of parameterized policies $\Pi$. 
Consequently, in Section~\ref{sec:theory} and~\ref{sec:algorithm}, we further develop the policy optimization loss algorithm by minimizing a certain divergence from $\pi$ to $\pi^*$. %

\begin{remark}[Optimal solution to unregularized Problem~\eqref{prob:bounded_ratio}]\label{rmk:unreg}
    Note that by taking $\lambda \rightarrow 0$ in Theorem~\ref{thr:conservative} and Theorem~\ref{thr:soft_guarantee}, one can obtain the optimal ratio and monotonic guarantees for the unregularized Problem~\eqref{prob:bounded_ratio}.
    In many cases, one can simplify the resulting optimal policy as $\pi^*(a|s)=[1+\epsilon\text{sign}(\tilde{A}_{\pi_0})]\pi_0(a|s)$ in~\eqref{eq:simple_solution}, where  $\mu_{\pi_0}(s)$ is the median of $Q_{\pi_0}(s, a)$ over $\pi_0(\cdot|s)$. %
    Such simplification holds if $\forall \,s,\,\exists\,\mu(s)$, such that
    \begin{equation}
         \mathbb{E}_{a\sim\pi_0(\cdot|s)}[\text{sign}(Q_{\pi_0}(s,a)-\mu(s))]=0. \label{eq:cond_sim}
    \end{equation}
    \looseness -1 Otherwise, the simplified $\pi^*$ can never be normalized.
    One valid case is a uniform density $\pi_0(\cdot|s)$ with continuous $\mathcal{A}$ and a $Q$-function $Q_{\pi_0}(s,a)$ which is smooth over $a$.
    However, there are also counterexamples.
    Consider, for instance, an MDP with a single state and a discrete action space $\mathcal A = \{a_1, a_2\}.$
    Assume that $Q_{\pi_0}(a_2)>Q_{\pi_0}(a_1)$, and $\pi_0(a_1)=\frac{1}{4}, \pi_0(a_2)=\frac{3}{4}$.
    Then for any $\mu\in \mathbb{R}$, condition~\eqref{eq:cond_sim} does not hold.
    While the simplified interpretation of $\pi^*$ is only valid in special cases, the result of \Cref{thr:conservative} still holds for arbitrarily small $\lambda > 0$ and general spaces (see Appendix~\ref{sec:proof_policy}).
\end{remark}

\subsection{Alternative Perspective: Minimizing Divergence from Optimal Policy}
\label{sec:theory}
In this section, we consider the policy optimization problem as minimizing the divergence to the optimal solution, instead of directly applying policy gradient methods.
Specifically, given the optimal policy obtained from \Cref{thr:conservative}, we can formulate policy optimization as 
\begin{align*}
    \min_{\pi_\theta\in\Pi} D(\pi_\theta, \pi^*),
\end{align*}
where $\pi_\theta$ is the parameterized policy, $D$ is a divergence function such as the KL-divergence, total variation (TV), etc.
Specifically, the TV (without $\frac{1}{2}$ multiplier) for each state can be expressed as
\begin{equation}
    \begin{split}
    D_{\theta}^{TV}(s):=\sum_a \left|\pi^*(a|s) - \pi_\theta(a|s)\right| =\mathbb{E}_{a\sim\pi_0(\cdot|s)}\left[\left|\frac{\pi^*(a|s)}{\pi_0(a|s)} - \frac{\pi_\theta(a|s)}{\pi_0(a|s)}\right|\right].
\end{split}\label{eq:tv_loss}
\end{equation}
We also consider an advantage-weighted TV (ATV) loss function defined as %
\begin{equation}
\begin{split}
&D_{\theta}^{ATV}(s):=\mathbb{E}_{a\sim\pi_0(\cdot|s)}\left[\left|\frac{\pi^*(a|s)}{\pi_0(a|s)} - \frac{\pi_\theta(a|s)}{\pi_0(a|s)}\right|\cdot\left|A_{\pi_0}\right|\right].
\end{split}\label{eq:true_loss}
\end{equation}
Notably, this divergence is directly correlated with the performance improvement of the parameterized policy $\pi_\theta$, as detailed in the following Corollary.
\begin{corollary}[Performance improvement guarantee with policy approximation error]
    Consider $D_{\theta}^{ATV}$ defined in~\eqref{eq:true_loss} with $\pi^*$ given from Theorem~\ref{thr:conservative}.
    Then it holds that
    \begin{align*}
        &\eta(\pi_\theta) \geq \eta(\pi_0) +  \mathbb{E}_{s\sim{d_{\pi_\theta}, a\sim \pi_0}(\cdot|s)}\left[\epsilon\tanh\left(\frac{\tilde{A}_{\pi_0}}{2\lambda}\right)\tilde{A}_{\pi_0} - D_{\theta}^{ATV}(s)\right],
    \end{align*}
    where $\tilde{A}_{\pi_0}$ abbreviates $\tilde{A}_{\pi_0}(s,a)$.
     \label{coro:soft_guarantee_true_l}
\end{corollary}

\looseness -1 The proof of Corollary~\ref{coro:soft_guarantee_true_l} is provided in~\ref{sec:proof_mono_true_l}.
Corollary~\ref{coro:soft_guarantee_true_l} shows that by minimizing the loss $D_{\theta}^{ATV}$ over the state distribution $d_{\pi_\theta}$, we can improve performance w.r.t.~policy $\pi_\theta$ as long as the parameterized policy class is sufficiently expressive.
However, minimizing $\mathbb{E}_{s\sim d_{\pi_\theta}}$ over the policy parameters $\theta$ is non-trivial due to the dependence of the state distribution on $\pi_\theta$.
Consequently, in practice we only optimize the expectation of the divergence over the old policy, leading to the objectives 
\begin{equation}
    J^{ATV}(\theta):=\mathbb{E}_{s\sim d_{\pi_0}}[D_\theta^{ATV}(s)],\quad J^{TV}(\theta):=\mathbb{E}_{s\sim d_{\pi_0}}[D_\theta^{TV}(s)].\label{eq:J_loss}
\end{equation} 
The following Corollary~\ref{coro:loss_tv} expresses a lower bound on the performance of the policy $\pi_{\theta}$ in terms of these quantities.

\begin{corollary}[Performance improvement guarantee with loss functions]
    With $B$ defined in Theorem~\ref{thr:soft_guarantee}, and $\Tilde{\delta}:=\max_{s} \mathbb{E}_{a\sim \pi_0(\cdot|s)}\left[\tanh\left(\frac{\tilde{A}_{\pi_0}}{2\lambda}\right)\tilde{A}_{\pi_0}\right]$, it holds that
    \begin{align*}
        \eta(\pi_\theta) 
        \geq\,\,&\eta(\pi_0) +  \epsilon B 
        - J^{ATV}(\theta) 
        - \frac{\gamma D^{ATV}_{\max}}{1-\gamma} J^{TV}(\theta) - \frac{ \gamma \epsilon D^{ATV}_{\max}}{(1-\gamma)^2} - \frac{\gamma\epsilon \Tilde{\delta}D^{TV}_{\max}}{(1-\gamma)^2},
    \end{align*}~\label{coro:loss_tv}
    where $D^{ATV}_{\max}:=\max_s D_{\theta}^{ATV}(s)$ and $D^{TV}_{\max}:=\max_s D_{\theta}^{TV}(s)$.
\end{corollary}

The proof of Corollary~\ref{coro:loss_tv} is detailed in Appendix~\ref{sec:proof_coro_loss_tv}.
Corollary~\ref{coro:loss_tv} decomposes the performance lower bound into a non-negative term, $\epsilon B$, and several negative terms which depend on the gap between the optimized policy $\pi_{\theta}$ and the policy $\pi^*$. The first two of these gap terms depend on $J^{ATV}(\theta)$ and $J^{TV}(\theta)$. Both of these quantities can be estimated from trajectories collected under $\pi_0$ and minimized by optimizing $\theta$, motivating a loss defined as a weighted combination of these terms. 
The other two gap terms are characterized by the worst case divergence of $\pi_{\theta}$ from $\pi^*$ over the state space through the quantities $D_{\max}^{ATV}$ and $D_{\max}^{TV}$. 
Though these quantities are generally not computable, they can be related to the expected losses over the distribution $d_{\pi_0}$ under additional assumptions (e.g., adequate state coverage under $d_{\pi_0}$). Notably, with perfect policy approximation ($\pi_{\theta}=\pi^*)$, it holds that $J^{ATV}(\theta) = J^{TV}(\theta)= D^{ATV}_{\max} = D^{TV}_{\max}=0,$ recovering the original monotonic performance guarantee of Theorem~\ref{thr:soft_guarantee}. 

In contrast, the bound from TRPO in Theorem 1 of~\cite{schulman2015trust} does not contain a positive term.
Besides, it penalizes the worst-case divergence between the updated policy and the current policy $\pi_0$ through $D_{TV}^{\max}(\pi,\pi_0)$, rather than the approximation error to an ideal solution $\pi^*$. Consequently, this negative term reflects the \emph{magnitude of the update} away from $\pi_0$. It vanishes only in the degenerate case $\pi=\pi_0$ (i.e., no policy change), and is generally nonzero whenever a nontrivial policy update occurs. %

\looseness -1 Corollary~\ref{coro:loss_tv} also motivates choosing a small $\epsilon$.
For small $\epsilon>0$, $\pi^*$ remains close to $\pi_0$ by construction, so matching $\pi^*$ typically requires only a small deviation from a realizable policy in the class (namely $\pi_0$), making the approximation error terms easier to control. 
Since $\pi_0$ is realizable in the policy class, the optimal values of $J^{ATV}(\theta)$ and $J^{TV}(\theta)$ approach zero as $\epsilon \to 0$. Therefore, under such regularity conditions, it holds that  $\frac{\gamma D_{\max}^{ATV}}{(1-\gamma)^2} + \frac{\gamma\tilde{\delta} D_{\max}^{TV}}{(1-\gamma)^2} \ll B$ for $\epsilon$ sufficiently small, therefore guaranteeing improvement.

\subsection{Bounded Policy Optimization}
\label{sec:algorithm}

In this section, we present the practical implementation of our algorithm.
As in PPO, we use a value network $V_\phi$ to estimate $A_{\pi_0}$.
Specifically, we estimate the return value $R_\phi(s, a)$ using generalized advantage estimation~\cite{schulman2015high}, and use it to update the value function by minimizing
\begin{align}
    J^{VF}(\phi):=\mathbb{E}_{s \sim d_{\pi_0}, a \sim \pi_0(\cdot|s)}[(sg(R_\phi(s, a)) - V_\phi(s))^2], \label{eq:vf_loss}
\end{align}
where $sg$ denotes stop gradient.
In addition, following Theorem~\ref{thr:conservative}, we further train a network $\mu_\psi$ to minimize the normalization loss
\begin{align}
    J^{MF}(\psi):=\mathbb{E}_{s \sim d_{\pi_0}, a \sim \pi_0(\cdot|s)}\left[\lambda g\left(\frac{sg(R_\phi(s,a)) - \mu_\psi(s)}{\lambda}\right)\right], \label{eq:mf_loss}
\end{align}
with $g$ defined in Theorem~\ref{thr:conservative}.
For numerical stability, we use the equivalent representation $g(x)=-\frac{x}{2}+\text{softplus}(x)$.
The practical loss function for $\theta$ uses the estimated advantage function to approximate $J^P(\theta):=J^{ATV}(\theta) + \alpha_1 J^{TV}(\theta)$ with $\alpha_1$ as a tunable weight:
\begin{equation}
    \begin{split}
        \hat{J}^{P}(\theta):=\mathbb{E}_{s \sim d_{\pi_0}, a \sim \pi_0(\cdot|s)}\left[\left|1 + \epsilon\tanh\left(\frac{\hat{A}_{\pi_0}}{2\lambda}\right) - \frac{\pi_\theta(a|s)}{\pi_0(a|s)}\right|\cdot sg\left(\left|R_\phi(s,a)-V_\phi(s)\right| + \alpha_1\right)\right],
    \end{split} \label{eq:loss}
\end{equation}
where $\hat{A}_{\pi_0}:=sg(R_\phi(s,a)-\mu_\psi(s))$.
Note that $\hat{J}^{P}(\theta)$ is not exactly $J^P$, but the gap can be controlled by minimizing the estimation error of $V_\phi$ and $\mu_\psi$.
Our final bounded policy optimization algorithm follows a PPO-style training procedure, summarized in Algorithm~\ref{algo:BPO}.

\begin{algorithm}[!hbtp]
    \caption{Bounded policy optimization (BPO)}\label{algo:BPO}
    \begin{algorithmic}[1]
    \STATE Initialize $\pi_\theta, V_\phi$, $\mu_\psi$, choose a sufficiently small $\lambda$
    \FOR{$i=1,2,...$} 
    \STATE Assign $\pi_0 \leftarrow \pi_\theta$
    \STATE Run $\pi_0$ for $N$ steps, and collect the dataset $\mathcal{D}:=\{s^j,a^j, R^j, \pi_0(a^j|s^j)\}_{j=1}^N$.
    \STATE Update $\theta, \phi, \psi$ by minimizing
    $$\hat{J}^{P}(\theta)+w_1 J^{VF}(\phi) + w_2 J^{MF}(\psi),$$ where $\hat{J}^{P}, J^{VF}, J^{MF}$ are defined in~\eqref{eq:loss},~\eqref{eq:vf_loss},~\eqref{eq:mf_loss}, and evaluated from $\mathcal{D}$.
    \ENDFOR
    \end{algorithmic}
\end{algorithm}

\subsection{Revisiting the PPO Objective}
\label{sec:revisitPPO}
In this section, we connect our theory and algorithmic framework to PPO~\cite{schulman2017proximal}.
In PPO, the following surrogate objective function is introduced
\begin{equation}
    \mathbb{E}_{s\sim d_{\pi_0},a\sim\pi_0}\left[\min\left\{\text{clip}\left(\rho,1-\epsilon,1+\epsilon\right)A_{\pi_0}, \rho A_{\pi_0}\right\}\right]:=J_{PPO}(\theta),\label{eq:ppo}
\end{equation}
where $\rho:=\frac{\pi_\theta(a|s)}{\pi_0(a|s)}$ is the ratio between the new and old policies, and $A_{\pi_0}$ denotes $A_{\pi_0}(s, a)$.

We observe a strong correlation between the BPO loss function and the PPO loss function.
To show this correlation, we first introduce an equivalent form of the PPO loss in the following proposition.
\begin{proposition}
    Optimizing the loss function $J_{PPO}(\theta)$ in~\eqref{eq:ppo} is equivalent to minimizing the following function
    \begin{align*}
    &J'(\theta):=\mathbb{E}_{s\sim{d_{\pi_0}}, a\sim \pi_0(\cdot|s)}\left[l'\left(\frac{\pi_\theta(a|s)}{\pi_0(a|s)}\right)\right],
    \end{align*}
    where
    \begin{equation*}
\begin{split}
        &l'(\rho ):=\begin{cases}|A_{\pi_0}|\cdot |\rho - (1+\epsilon\cdot\text{sign}(A_{\pi_0}) )|, & [\rho - (1+\epsilon\cdot\text{sign}(A_{\pi_0}))]\cdot A_{\pi_0} \leq 0, \\[2pt]
0, & \text{Otherwise}. 
\end{cases}
\end{split}
\end{equation*}\label{prop:ppo_equ}
\end{proposition}
Intuitively, this equivalence follows from the fact that adding or subtracting a constant from the objective function does not change the optimal solution.  
This is illustrated in Fig.~\ref{fig:bpo_ppo} (a vs b).
A proof is provided in Appendix~\ref{app:proof_ppo_equ}.

On the other hand,  as $\lambda \to 0$, the loss $J^{ATV}$ in~\eqref{eq:J_loss} can, in many cases (Remark~\ref{rmk:unreg}), be expressed as
\begin{align*}
J^{ATV}(\theta) \approx \mathbb{E}_{s\sim d_{\pi_0},a\sim \pi(\cdot|s)}[l^{BPO}(\rho)],\quad \text{where}\,\,
    l^{BPO}(\rho)=|A_{\pi_0}|\cdot|\rho - (1+\epsilon\cdot\text{sign}(\Tilde{A}_{\pi_0}) )|.
\end{align*}
Thus, the loss $l^{BPO}$ resembles the PPO loss $l'$ in Proposition~\ref{prop:ppo_equ} when $[\rho - (1+\epsilon\cdot\text{sign}(A_{\pi_0}))]\cdot A_{\pi_0} \leq 0$, as detailed in Section~\ref{sec:overview}. %
For $[\rho - (1+\epsilon\cdot\text{sign}(A_{\pi_0}))]\cdot A_{\pi_0} > 0$, BPO penalizes the policy for deviating from the original policy, which encourages satisfaction of the bounded-ratio constraints.
This is also partially addressed by the zero gradient of PPO and the target KL divergence mechanism \citep{serrano2023skrl}, which slows down the update of the new policy if it deviates too far from the original policy (i.e., if the KL divergence between the two surpasses the target KL divergence).
Recent PPO variants also utilize similar ideas by introducing negative gradients when $[\rho - (1+\epsilon\cdot\text{sign}(A_{\pi_0}))]\cdot A_{\pi_0} > 0$~\cite{wang2020truly,xie2024simple}, which can be theoretically justified by our framework.
Although the exact optimization dynamics of BPO and PPO differ, both follow a common principle: drive the policy ratio from $1$ toward the (approximate) analytical optimum of BRRL and then stop.
This offers a key insight into the underlying success of PPO-based methods.

\subsection{Extension to LLM Fine-Tuning}
\label{sec:gbpo}
\looseness -1 In the context of LLM fine-tuning, training an additional critic can be computationally expensive.
This challenge motivates the design of Group Relative Policy Optimization (GRPO)~\cite{shao2024deepseekmath}, which estimates advantages relative to a group of concurrent samples rather than utilizing an auxiliary value network. 
Building on this idea, we introduce Group-relative Bounded Policy Optimization (GBPO), an extension of BPO derived from Theorem~\ref{thr:conservative}.
Specifically, for a given prompt $q$, the model generates a group of sampled outcomes $\{o_1, o_2, \dots, o_G\}$. 
A reward model then assigns a score to each output, denoted by $\mathbf{R}=\{r_1, r_2, \dots, r_G\}$. 
As in standard GRPO, we estimate advantages using z-scores $A_i := \frac{r_i - \text{mean}(\mathbf{R})}{\text{std}(\mathbf{R})}$. 
As noted in Remark~\ref{rmk:unreg}, when the regularization parameter $\lambda$ is small, the implicit baseline $\mu_{\pi_0}(q)$ converges to the median of the Q-values. 
We therefore also estimate the \textit{median-advantage} as $\tilde{A}_i := \frac{r_i - \text{median}(\mathbf{R})}{\text{std}(\mathbf{R})}$. 
The GBPO objective function is then defined as:
{\allowdisplaybreaks
\begin{align*}
    &\hat{J}^{P}(\theta) = \mathbb{E}_{\substack{q \sim P(\mathcal{Q}) \\ o_1,\dots,o_G \sim \pi_0(\cdot|q)}} \Bigg[ \frac{1}{G} \sum_{i=1}^G \frac{1}{|o_i|} \sum_{t=1}^{|o_i|} \left| 1 + \epsilon\tanh\left(\frac{\tilde{A}_{i,t}}{2\lambda}\right) - \frac{\pi_\theta(o_{i,t}|q,o_{i,<t})}{\pi_0(o_{i,t}|q,o_{i,<t})} \right|\cdot |A_{i,t}| \Bigg],
\end{align*}
}where $t$ denotes the token index, and $\mathcal{Q}$ is the question set. In scenarios where a reward is only provided at the end of the sequence, the step-dependent advantages $A_{i,t}$ and $\tilde{A}_{i,t}$ are equal to the sequence-level $A_i$ and $\tilde{A}_i$, respectively. If per-step scores are available, these advantages can be estimated token-wise following the approach in \cite{shao2024deepseekmath}.

\subsection{Asymmetric Ratio Constraints and Cross Entropy Method}
\label{sec:cem}
In this section, we generalize Theorem~\ref{thr:conservative} to asymmetric ratio constraints.
Similar to~\eqref{eq:conservative}, we consider the regularized problem with general ratio boundaries $\forall s, a,\,\,c_l\leq\frac{\pi(a|s)}{\pi_0(a|s)}\leq c_h$, with $c_l<1<c_h$ 
\begin{equation}
\begin{split}
    &\max_{\pi}  \, L_{\pi_0}(\pi) - \lambda \mathbb{E}_{s\sim d_{\pi_0},a\sim \pi_0}\left[H'\left(\frac{\pi(a|s)}{\pi_0(a|s)}\right)\right], \\
    \text{where} \quad& H'(\rho):=(\rho - c_l) \log(\rho - c_l)  + (c_h - \rho)\log(c_h - \rho) + \rho\log\frac{c_h-1}{1-c_l}. \label{prob:unbalanced}
\end{split}
\end{equation}
Here, the regularizer $H'$ still takes its minimum at $\rho=1$, and provides log barriers for the asymmetric constraints $c_l\leq\rho\leq c_h$.
The optimal solution is detailed in the following Corollary.
\begin{corollary}(Asymmetric optimal policy)
    The optimal policy $\pi^*$ of the problem~\eqref{prob:unbalanced} satisfies
    \begin{align*}
        \frac{\pi^*(a|s)}{\pi_0(a|s)}=c_l + \frac{c_h - c_l}{1 + \frac{c_h-1}{1-c_l}\exp(-\tilde{A}'_{\pi_0}/{\lambda})},\,\,\tilde{A}'_{\pi_0}:=Q_{\pi_0}(s,a) - \mu'_{\pi_0}(s),
    \end{align*}
    where $\mu'_{\pi_0}(s)$ is called the soft-$\frac{c_h-1}{c_h-c_l}$-quantile satisfying
    \begin{align*}
        \mu'_{\pi_0}(s) = \arg\min_{\mu(s)} \mathbb{E}_{a\sim\pi_0(\cdot|s)}\left[ g'\left(\frac{Q_{\pi_0}(s, a) - \mu(s)}{\lambda}\right)\right],
    \end{align*}
    where $g':\mathbb{R}\rightarrow \mathbb{R}_{\geq 0},g'(x)=\ln\left(e^{\frac{c_h-1}{c_h-c_l}x} + \frac{c_h-1}{1-c_l}e^{-\frac{1-c_l}{c_h-c_l}x}\right)$. \label{coro:unbalance}
\end{corollary}
Similar to Theorem~\ref{thr:conservative}, these results are closely related to policy optimization with asymmetric clip ratios~\cite{wang2025aspo,xi2025bapo}.
A monotonic performance guarantee similar to Theorem~\ref{thr:soft_guarantee} is provided in Appendix~\ref{app:proof_asy_guarantee}.
Moreover, when $\lambda\rightarrow0$, we also have $\mu'_{\pi_0}(s)$ the exact $\frac{c_h-1}{c_h-c_l}$-quantile in many cases, similar to Remark~\ref{rmk:unreg}.
Notably, when $c_l=0$, $\lambda\rightarrow 0$, and the $\frac{c_h-1}{c_h-c_l}$-quantile exists for $Q_{\pi_0}(s, a)$, we have $\pi^*(a|s)=c_h \pi_0(a|s)$ for $Q_{\pi_0}(s,a)>\mu'_{\pi_0}(s)$ and $0$ otherwise.
This recovers a cross-entropy method (CEM) when $\pi_0$ is uniform, where the optimal solution at each iteration assigns non-zero probability to the top fraction of $\frac{1-c_l}{c_h-c_l}$ samples.

\section{Experiments}
\label{sec:exp}
\looseness -1 In this section, we present extensive experiments to validate the proposed BPO algorithm. 
We first benchmark its performance against PPO across standard MuJoCo and Atari environments (Section~\ref{sec:exp_classic}). 
To assess scalability, we evaluate BPO within NVIDIA IsaacLab~\cite{mittal2025isaac}, a high-throughput simulation platform capable of simulating \emph{thousands of} parallel environments for real-world robotic policy training. 
Furthermore, we apply our GBPO variant to LLM fine-tuning tasks, and compare it directly against GRPO (Section~\ref{sec:exp_LLM}). 
Then, we dive deeper into the analysis of the ratio statistics during training, connecting it to the performance gap between BPO and PPO.
Finally, we conduct an ablation study to analyze the sensitivity of training performance to key components, including the loss function, the $\lambda$ parameter, and various loss coefficients.
All hyperparameters are detailed in Appendix~\ref{sec:hyper}

\subsection{Benchmarking with Classical Environments}
\label{sec:exp_classic}

\looseness -1 We compare the performance of BPO and PPO in classical environments. 
For these experiments, BPO was implemented within the Stable Baselines3 framework, with hyperparameters for all baseline algorithms sourced from RL-Zoo~\cite{rl-zoo3}.
As shown in Figure~\ref{fig:exp_classic}, BPO performs competitively with or superior to PPO across a range of classical benchmarks. 
Specifically, in MuJoCo tasks, BPO achieves clear performance gains in the Ant-v4, Hopper-v4, and Humanoid-v4 environments.
Training on Humanoid-v4 exhibits high variance for BPO, characterized by significant performance divergence across random seeds. 
Both PPO and BPO struggle to achieve peak performance in this environment, primarily due to limited sample efficiency.
However, as demonstrated in Section \ref{sec:exp_isaaclab}, both methods successfully solve more complex humanoid tasks when provided with sufficient samples. 
In Atari benchmarks, BPO generally matches PPO’s performance, notably outperforming it in the Asterix environment. 

We report our benchmarking results against off-policy baselines in MuJoCo and Atari environments in Table~\ref{tab:exp_classic}.
While SAC~\cite{haarnoja2018soft} outperforms both PPO and BPO in the Ant-v4 and Humanoid-v4 tasks, it fails to achieve competitive results in Swimmer-v4. 
In contrast, BPO consistently outperforms PPO in the Ant-v4, Humanoid-v4, and Hopper-v4 environments while remaining competitive in Swimmer-v4.
Both BPO and PPO consistently outperform DQN~\cite{mnih2013playing} in Atari benchmarks.

\begin{figure*}[t] 
    
    \centering %
    
    \includegraphics[width=1.00\textwidth]{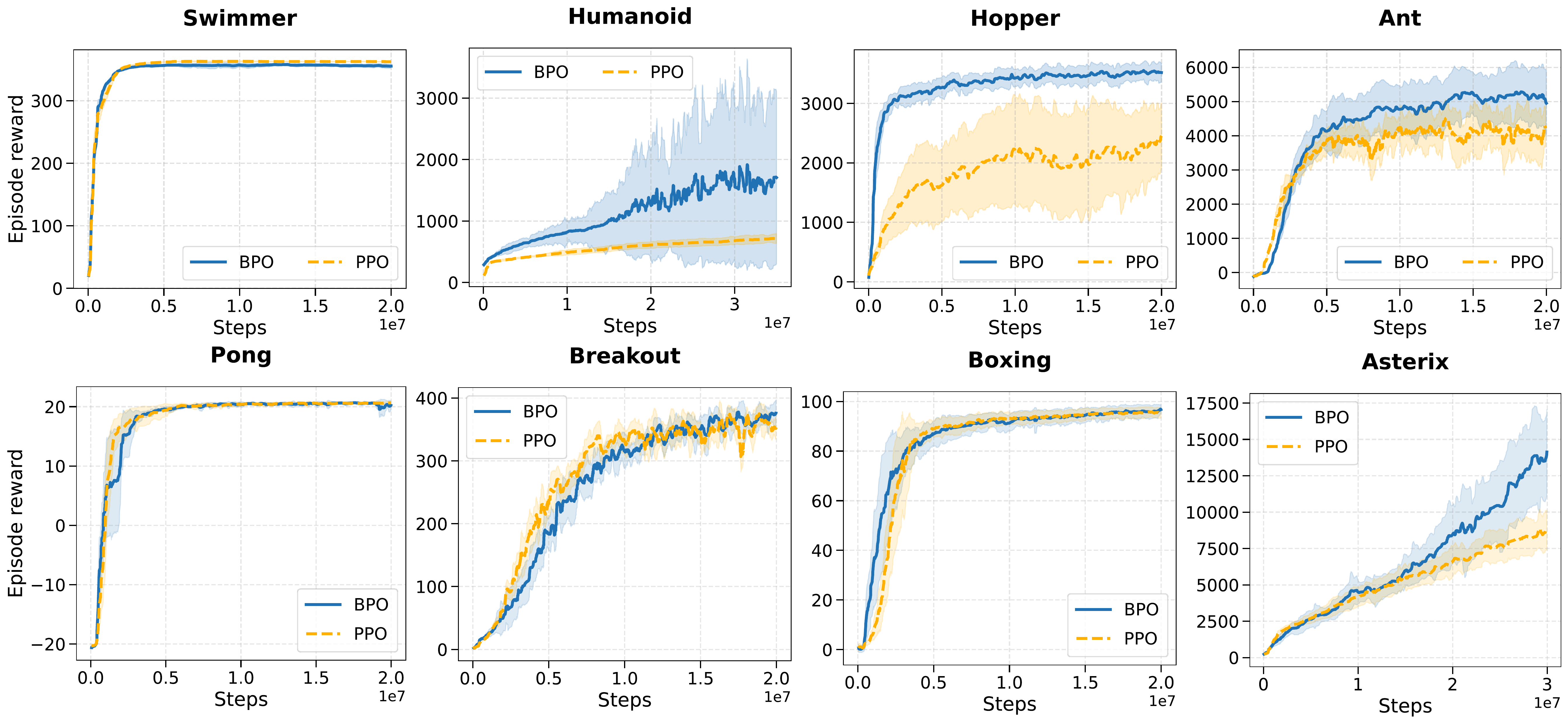} 
    
    \caption{\looseness -1 BPO versus PPO on MuJoCo and Atari environments.
    Shaded regions represent the standard deviation across 10 random seeds.
    In most environments, BPO matches or outperforms PPO.
    }
    \label{fig:exp_classic}
\end{figure*}

\begin{table}[!hbtp]
\footnotesize
\centering
\caption{Comparison of converged total rewards between BPO, PPO, and off-policy algorithms.
Bolded and underlined numbers indicate the highest and second-highest results across all tested algorithms.
For AsterixNoFrameskip-v4, the algorithms are evaluated after the same wall-clock time (12h).
Rewards in other environments are evaluated after convergence.
}
\begin{tabular}{cccc|cccc}
\hline
Mujoco Envs & BPO             & PPO            & SAC             & Atari Envs             & BPO & PPO   & DQN   \\ \hline
Ant-v4      & \ul {4871.4}    & 4230.1         & \textbf{6161.8} & BreakoutNoFrameskip-v4 &  \textbf{374.6}   & \ul{360.4} & 252.5 \\
Humanoid-v4 & \ul {1570.4}    & 781.3          & \textbf{6806.4} & PongNoFrameskip-v4     &  \textbf{20.6}   & \textbf{20.6}  & \textbf{20.6}  \\
Hopper-v4   & \textbf{3505.1} & 2497.7         & \ul {3015.1}    & BoxingNoFrameskip-v4   &   \ul {94.7}  & \textbf{95.7}  & 92.5  \\
Swimmer-v4  & \ul{354.6}     & \textbf{362.4} & 102.7   & AsterixNoFrameskip-v4   &  \textbf{11247.9}  &  \ul{9471.5}   &  7122.8  \\ \hline
\end{tabular}
\label{tab:exp_classic}
\end{table}

\subsection{Benchmarking with IsaacLab Environments}
\label{sec:exp_isaaclab}

In this section, we evaluate the scalability and performance of BPO relative to PPO within the IsaacLab simulation platform. 
We focus on four challenging locomotion tasks on rough terrain: Go1-rough, Anymal-C-rough, G1-rough, and H1-rough, which require the agents (quadrupeds like Unitree Go1 and Anymal-C or humanoids like Unitree G1 and H1) to maintain stable gaits while tracking target velocities across rough surfaces. 
Both BPO and PPO were implemented using the RSL-RL framework, utilizing a large-scale parallelization of 4,096 environments per task. 

\begin{figure*}[t] 
    
    \centering %
    
    \includegraphics[width=0.98\textwidth]{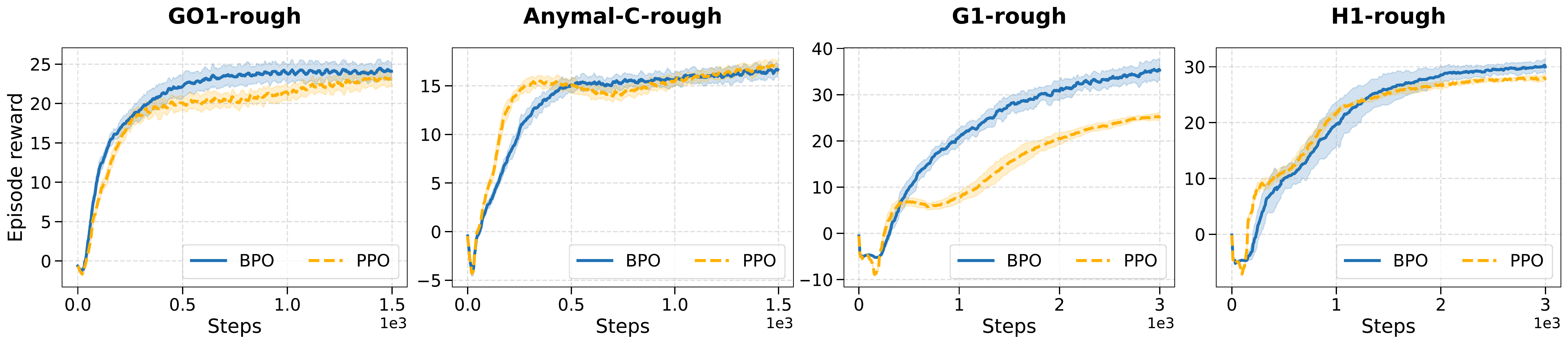} 
    
    \caption{BPO versus PPO on IsaacLab environments.
    Shaded regions represent standard deviation across 5 random seeds.
    BPO substantially outperforms the baseline on challenging humanoid locomotion tasks while exhibiting more stable training dynamics.
    }
    \label{fig:exp_isaaclab}
\end{figure*}

The results in Figure~\ref{fig:exp_isaaclab} demonstrate that BPO is highly effective in complex robotic locomotion tasks. 
In particular, on G1-rough, BPO significantly outperforms the baseline to reach a higher performance ceiling. 
For the Go1-rough and H1-rough environment, BPO also slightly exceeds the final performance of PPO. 
Notably, across all four benchmarks, BPO exhibits enhanced training stability and smoother dynamics compared to the PPO baseline.

\begin{wrapfigure}[18]{r}{0.6\textwidth}
    \includegraphics[width=0.6\textwidth]{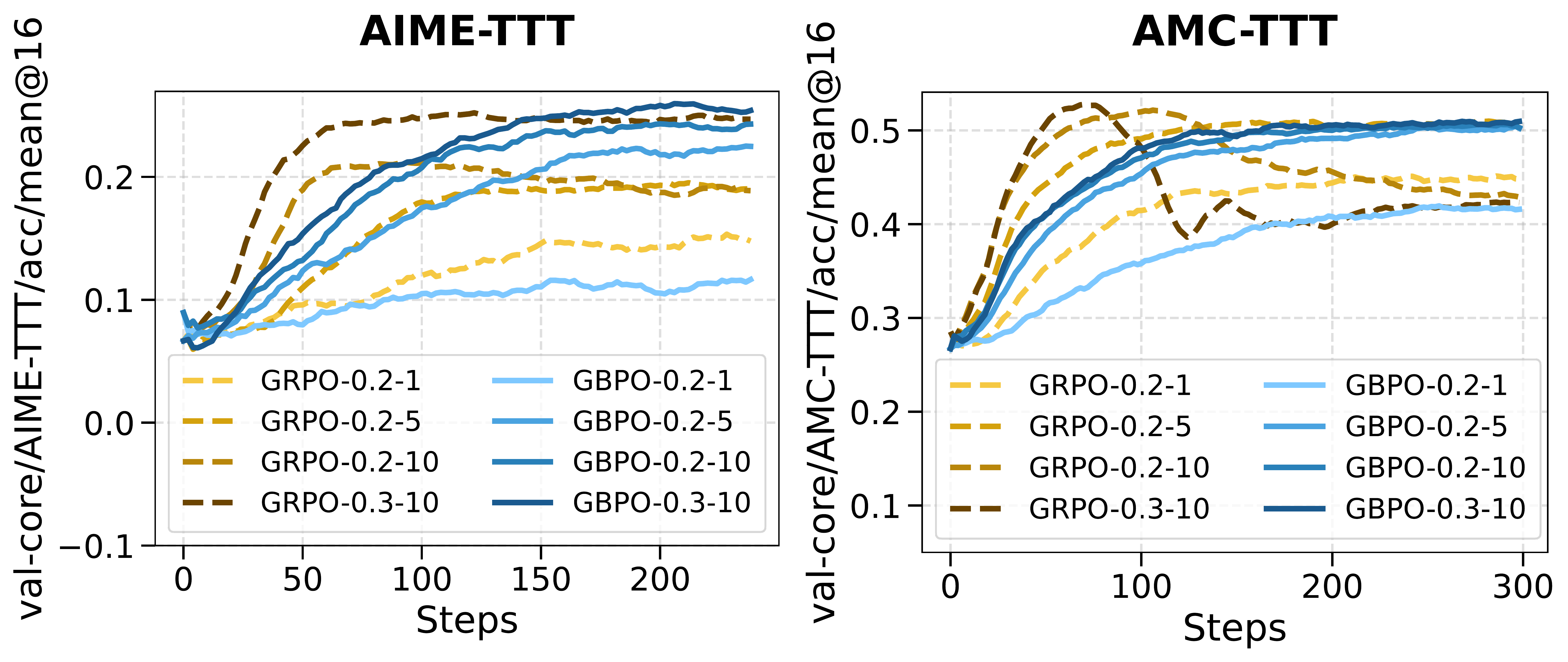} 
    \caption{\looseness -1 Performance of GRPO (green) and GBPO (blue) for fine-tuning Qwen2.5-Math-1.5B on AIME-TTT and AMC-TTT benchmarks.
    In the legend, the first and second numbers denote the clip ratio and the number of epochs, respectively.
    }\label{fig:exp_TTRL}
\end{wrapfigure}

\subsection{LLM Fine-Tuning with GBPO}
\label{sec:exp_LLM}
\looseness -1 We further evaluate GBPO against GRPO for large language model fine-tuning (Section~\ref{sec:gbpo}). 
Specifically, we conduct experiments in the Test-Time Reinforcement Learning (TTRL,~\cite{zuo2025ttrl}) framework, fine-tuning the Qwen2.5-Math-1.5B model with GBPO and GRPO on the AIME-TTT and AMC-TTT benchmarks, and then compare their reasoning performance.
The empirical results, illustrated in Figure~\ref{fig:exp_TTRL}, reveal that GBPO can maintain performance gains as the number of training epochs and clip ratio increase. 
Conversely, GRPO exhibits instability under these conditions.
These findings highlight GBPO’s potential as a more robust and stable alternative for the fine-tuning of large-scale models.

\subsection{Ratio Statistics Analysis}

\begin{figure*}[!htbp] 
    
    \centering %
    
    \includegraphics[width=1.0\textwidth]{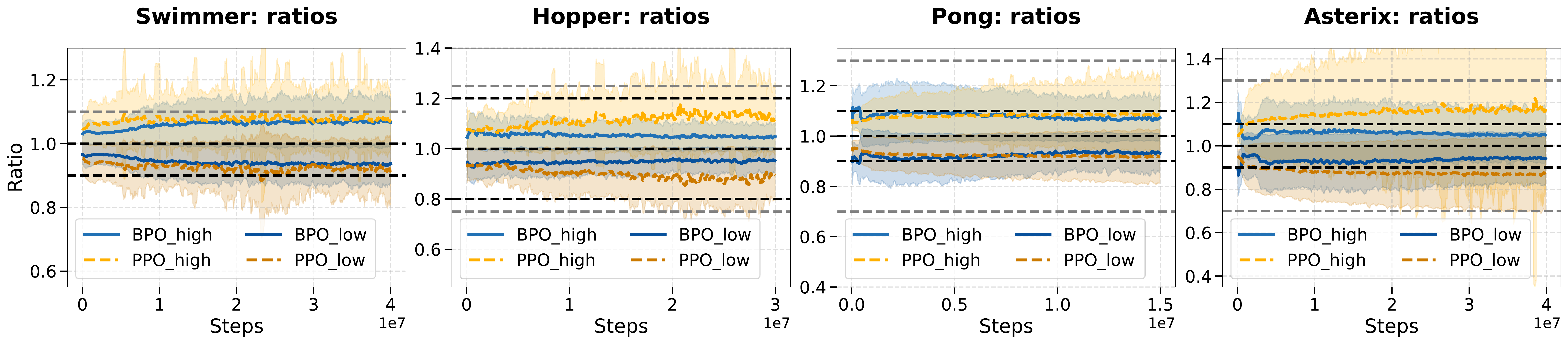} 
    
    \caption{Analysis of ratio statistics.
    During the training process, we draw statistics of ratios ($\pi(a|s)/\pi_0(a|s)$) above and below 1.0 separately, corresponding to BPO/PPO\_high and BPO/PPO\_low.
    Solid lines and shaded regions represent the mean and standard deviations across 5 random seeds.
    Dashed black lines highlight $1.0$ and clipped ranges for PPO; dashed gray lines show clipped ranges for BPO.
    }
    \label{fig:exp_analysis}
\end{figure*}

We analyze the statistics of importance weights (ratio $\pi(a|s)/\pi_0(a|s)$) during the training process.
In MuJoCo environments (using the stable-baselines3 implementation), BPO maintains more stable ratio distributions than PPO, as illustrated in Figure~\ref{fig:exp_analysis}. 
This difference in stability is more obvious in environments where BPO outperforms PPO (e.g., Hopper and Asterix).
\begin{wrapfigure}[15]{r}{0.6\textwidth}
    \includegraphics[width=0.6\textwidth]{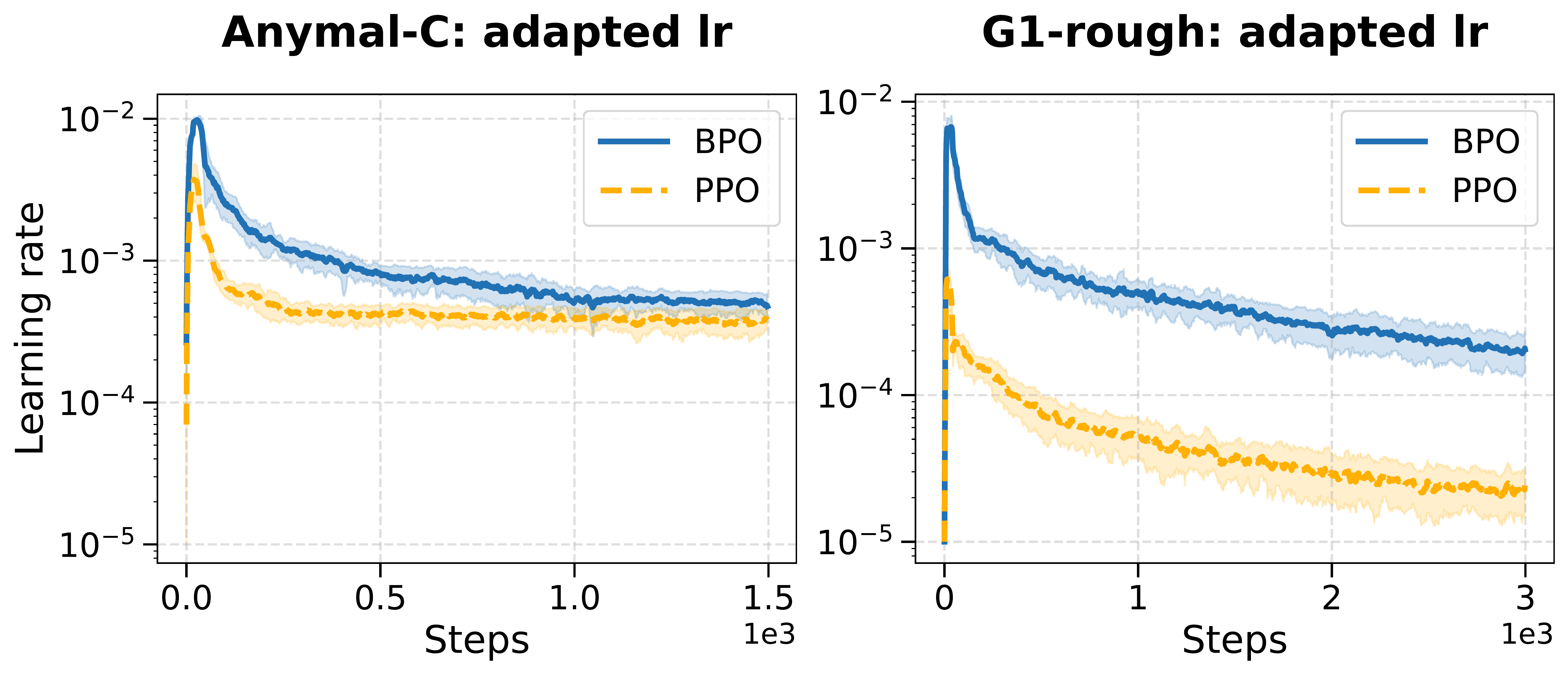} 
    \caption{Adapted learning rates to match the target KL divergence in RSL-RL implementation.}\label{fig:exp_adp_lr}
\end{wrapfigure}

In IsaacLab environments (utilizing RSL-RL), learning rates are dynamically adjusted to maintain a target KL divergence. 
As shown in Figure~\ref{fig:exp_adp_lr}, the adapted learning rates for PPO are often lower than those for BPO, suggesting more aggressive ratio updates that surpass the target KL divergence more frequently.
The scales of the learning rates differ more in tasks where BPO shows a clear performance improvement (e.g., G1-rough).
These findings suggest a strong correlation between the stability of ratio distributions and overall algorithmic performance. 
By effectively enforcing this stability, BPO allows for more stable performance improvement.

\subsection{Ablation Study}
\label{sec:exp_ablation}

This section presents an ablation study of the impact of the value function, loss function, $\lambda$, and the coefficient of TV loss on the performance of the policy, within the G1-rough environment.

\begin{figure}[!hbtp] 
    
    \centering %
    
    \includegraphics[width=1.0\textwidth]{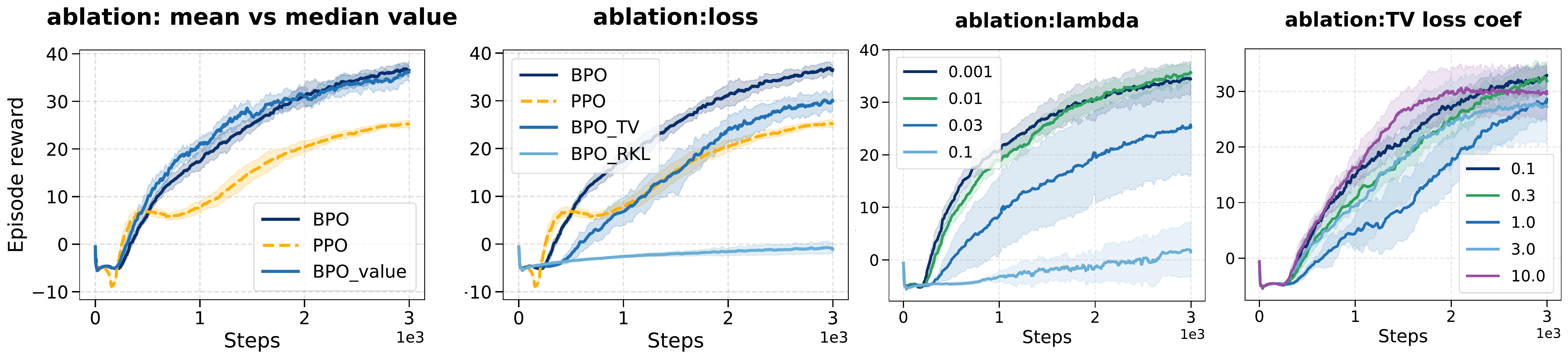} 
    
    \caption{\looseness -1 Ablation study of BPO components in G1-rough environment.
    Shaded regions represent standard deviation across 10 random seeds.
    From left to right: ablation of mean vs median value functions, loss functions, regularization weight $\lambda$ and total variation (TV) weight $\alpha_1$.
    The mean-value variant achieves performance comparable to the median-value version (BPO\_value).
    The advantage-weighted total variation (ATV) loss provides better performance compared to TV and reverse KL divergence (RKL).
    In general, smaller $\lambda$ values lead to better performance.
    In practice, including the TV loss does not improve results.
    }
    \label{fig:ablation_loss}
\end{figure}

\textbf{Mean vs median value function.}
We evaluate the performance of the algorithm by substituting median advantages $\Tilde{A}_{\pi_0}$ with the mean advantage ${A}_{\pi_0}$. 
As illustrated in the left panel of Figure~\ref{fig:ablation_loss}, this simplification achieves performance comparable to the original BPO. 
This robustness likely stems from the low practical differences between median and mean values, caused by the specific return distribution and inherent value estimation errors. 
These results also suggest that this median-to-mean value simplification offers a compelling alternative when the computational overhead of learning the median value is high.

\textbf{Divergence function ablation.}
As illustrated in the middle-left panel of Figure~\ref{fig:ablation_loss}, the ATV loss yields superior performance in the G1-rough environment. 
While the standard TV loss facilitates some learning, it fails to match the asymptotic performance of ATV. 
Conversely, KL divergence proves ineffective and fails to achieve successful policy convergence. 

\textbf{Sensitivity to $\lambda$.}
\looseness -1 We conduct a hyperparameter sweep for $\lambda$ in the G1-rough environment. As shown in the middle-right panel of Figure~\ref{fig:ablation_loss} , smaller values of $\lambda$ generally lead to strong performance. 
Specifically, increasing $\lambda$ from $10^{-3}$ to $10^{-2}$ may slightly improve asymptotic performance, but at the cost of a reduced convergence rate. 
Conversely, excessively large values of $\lambda$ prevent the learning process entirely.

\textbf{Impact of TV loss regularization.}
We study the effect of the TV loss coefficient by incrementally increasing its weight relative to the ATV loss in the G1-rough environment. 
Although Corollary~\ref{coro:loss_tv} suggests that both terms contribute to performance gains, the results in the right panel of Figure~\ref{fig:ablation_loss} indicate that explicitly adding a TV loss component does not improve performance in practice. 

\section{Conclusion}
We introduced Bounded Ratio Reinforcement Learning (BRRL), a framework 
for %
policy optimization under bounded ratio constraints. We showed that the underlying optimization problem admits an analytic solution. 
Our main finding is that this optimal solution allows interpreting the PPO loss from a new perspective, connects to the cross-entropy method (CEM), and motivates a \emph{theoretically grounded} variant, Bounded Policy Optimization (BPO). 
Empirically, BPO is consistently effective across a broad range of tasks, including robotic control and large-model fine-tuning. 
Despite the extensive evaluation with standard RL benchmarks, extending the experiments towards a broader range of LLM fine-tuning tasks remains a compelling future direction. 
Other future research directions include enhancing sample efficiency via advanced exploration, extending the framework to constrained MDPs, and adapting the algorithm for fine-tuning generative policies.

\section{Acknowledgment}
This work is in part supported by the Hasler Foundation
("Learn to learn safely" project, grant number: 21039), 
Swiss National Science Foundation under NCCR Automation, grant agreement 51NF40 180545, and the ETH AI Center.

\bibliography{example_paper}

\begin{thebibliography}{10}

\bibitem{akgul2025overcoming}
Abdullah Akg{\"u}l, Gulcin Baykal, Manuel Hau{\ss}mann, and Melih Kandemir.
\newblock Overcoming non-stationary dynamics with evidential proximal policy optimization.
\newblock {\em arXiv preprint arXiv:2503.01468}, 2025.

\bibitem{andrychowicz2020learning}
OpenAI:~Marcin Andrychowicz, Bowen Baker, Maciek Chociej, Rafal Jozefowicz, Bob McGrew, Jakub Pachocki, Arthur Petron, Matthias Plappert, Glenn Powell, Alex Ray, et~al.
\newblock Learning dexterous in-hand manipulation.
\newblock {\em The International Journal of Robotics Research}, 39(1):3--20, 2020.

\bibitem{cobbe2021phasic}
Karl~W Cobbe, Jacob Hilton, Oleg Klimov, and John Schulman.
\newblock Phasic policy gradient.
\newblock In {\em International Conference on Machine Learning}, pages 2020--2027. PMLR, 2021.

\bibitem{doering2026approximate}
Leif Doering, Daniel Schmidt, Moritz Melcher, Sebastian Kassing, Benedikt Wille, Tilman Aach, and Simon Weissmann.
\newblock An approximate ascent approach to prove convergence of ppo.
\newblock {\em arXiv preprint arXiv:2602.03386}, 2026.

\bibitem{engstrom2020implementation}
Logan Engstrom, Andrew Ilyas, Shibani Santurkar, Dimitris Tsipras, Firdaus Janoos, Larry Rudolph, and Aleksander Madry.
\newblock Implementation matters in deep policy gradients: A case study on ppo and trpo.
\newblock {\em arXiv preprint arXiv:2005.12729}, 2020.

\bibitem{fakoor2020p3o}
Rasool Fakoor, Pratik Chaudhari, and Alexander~J Smola.
\newblock P3o: Policy-on policy-off policy optimization.
\newblock In {\em Uncertainty in artificial intelligence}, pages 1017--1027. PMLR, 2020.

\bibitem{haarnoja2018soft}
Tuomas Haarnoja, Aurick Zhou, Pieter Abbeel, and Sergey Levine.
\newblock Soft actor-critic: Off-policy maximum entropy deep reinforcement learning with a stochastic actor.
\newblock In {\em International conference on machine learning}, pages 1861--1870. PMLR, 2018.

\bibitem{kobayashi2021proximal}
Taisuke Kobayashi.
\newblock Proximal policy optimization with relative pearson divergence.
\newblock In {\em 2021 IEEE International Conference on Robotics and Automation (ICRA)}, pages 8416--8421. IEEE, 2021.

\bibitem{kullback1951information}
Solomon Kullback and Richard~A Leibler.
\newblock On information and sufficiency.
\newblock {\em The annals of mathematical statistics}, 22(1):79--86, 1951.

\bibitem{landau1980statistical}
L.~D. Landau, E.~M. Lifshitz, and L.~P. Pitaevskii.
\newblock {\em Statistical Physics: Theory of the Condensed State}, volume~9 of {\em Course of Theoretical Physics}.
\newblock Butterworth-Heinemann, Oxford, 1980.

\bibitem{lee2020learning}
Joonho Lee, Jemin Hwangbo, Lorenz Wellhausen, Vladlen Koltun, and Marco Hutter.
\newblock Learning quadrupedal locomotion over challenging terrain.
\newblock {\em Science robotics}, 5(47):eabc5986, 2020.

\bibitem{liu2019neural}
Boyi Liu, Qi~Cai, Zhuoran Yang, and Zhaoran Wang.
\newblock Neural trust region/proximal policy optimization attains globally optimal policy.
\newblock {\em Advances in neural information processing systems}, 32, 2019.

\bibitem{miki2022learning}
Takahiro Miki, Joonho Lee, Jemin Hwangbo, Lorenz Wellhausen, Vladlen Koltun, and Marco Hutter.
\newblock Learning robust perceptive locomotion for quadrupedal robots in the wild.
\newblock {\em Science robotics}, 7(62):eabk2822, 2022.

\bibitem{milosevic2025central}
Nikola Milosevic, Johannes M{\"u}ller, and Nico Scherf.
\newblock Central path proximal policy optimization.
\newblock {\em arXiv preprint arXiv:2506.00700}, 2025.

\bibitem{mittal2025isaac}
Mayank Mittal, Pascal Roth, James Tigue, Antoine Richard, Octi Zhang, Peter Du, Antonio Serrano-Mu{\~n}oz, Xinjie Yao, Ren{\'e} Zurbr{\"u}gg, Nikita Rudin, et~al.
\newblock Isaac lab: A gpu-accelerated simulation framework for multi-modal robot learning.
\newblock {\em arXiv preprint arXiv:2511.04831}, 2025.

\bibitem{mnih2013playing}
Volodymyr Mnih, Koray Kavukcuoglu, David Silver, Alex Graves, Ioannis Antonoglou, Daan Wierstra, and Martin Riedmiller.
\newblock Playing atari with deep reinforcement learning.
\newblock {\em arXiv preprint arXiv:1312.5602}, 2013.

\bibitem{ouyang2022training}
Long Ouyang, Jeffrey Wu, Xu~Jiang, Diogo Almeida, Carroll Wainwright, Pamela Mishkin, Chong Zhang, Sandhini Agarwal, Katarina Slama, Alex Ray, et~al.
\newblock Training language models to follow instructions with human feedback.
\newblock {\em Advances in neural information processing systems}, 35:27730--27744, 2022.

\bibitem{qi2026rethinking}
Penghui Qi, Xiangxin Zhou, Zichen Liu, Tianyu Pang, Chao Du, Min Lin, and Wee~Sun Lee.
\newblock Rethinking the trust region in llm reinforcement learning.
\newblock {\em arXiv preprint arXiv:2602.04879}, 2026.

\bibitem{radosavovic2024real}
Ilija Radosavovic, Tete Xiao, Bike Zhang, Trevor Darrell, Jitendra Malik, and Koushil Sreenath.
\newblock Real-world humanoid locomotion with reinforcement learning.
\newblock {\em Science Robotics}, 9(89):eadi9579, 2024.

\bibitem{rl-zoo3}
Antonin Raffin.
\newblock Rl baselines3 zoo.
\newblock \url{https://github.com/DLR-RM/rl-baselines3-zoo}, 2020.

\bibitem{rubinstein1999cross}
Reuven Rubinstein.
\newblock The cross-entropy method for combinatorial and continuous optimization.
\newblock {\em Methodology and computing in applied probability}, 1(2):127--190, 1999.

\bibitem{schulman2015trust}
John Schulman, Sergey Levine, Pieter Abbeel, Michael Jordan, and Philipp Moritz.
\newblock Trust region policy optimization.
\newblock In {\em International conference on machine learning}, pages 1889--1897. PMLR, 2015.

\bibitem{schulman2015high}
John Schulman, Philipp Moritz, Sergey Levine, Michael Jordan, and Pieter Abbeel.
\newblock High-dimensional continuous control using generalized advantage estimation.
\newblock {\em arXiv preprint arXiv:1506.02438}, 2015.

\bibitem{schulman2017proximal}
John Schulman, Filip Wolski, Prafulla Dhariwal, Alec Radford, and Oleg Klimov.
\newblock Proximal policy optimization algorithms.
\newblock {\em arXiv preprint arXiv:1707.06347}, 2017.

\bibitem{schwarke2025rsl}
Clemens Schwarke, Mayank Mittal, Nikita Rudin, David Hoeller, and Marco Hutter.
\newblock Rsl-rl: A learning library for robotics research.
\newblock {\em arXiv preprint arXiv:2509.10771}, 2025.

\bibitem{serrano2023skrl}
Antonio Serrano-Munoz, Dimitrios Chrysostomou, Simon B{\o}gh, and Nestor Arana-Arexolaleiba.
\newblock skrl: Modular and flexible library for reinforcement learning.
\newblock {\em Journal of Machine Learning Research}, 24(254):1--9, 2023.

\bibitem{shao2024deepseekmath}
Zhihong Shao, Peiyi Wang, Qihao Zhu, Runxin Xu, Junxiao Song, Xiao Bi, Haowei Zhang, Mingchuan Zhang, YK~Li, Yang Wu, et~al.
\newblock Deepseekmath: Pushing the limits of mathematical reasoning in open language models.
\newblock {\em arXiv preprint arXiv:2402.03300}, 2024.

\bibitem{silver2017mastering}
David Silver, Julian Schrittwieser, Karen Simonyan, Ioannis Antonoglou, Aja Huang, Arthur Guez, Thomas Hubert, Lucas Baker, Matthew Lai, Adrian Bolton, et~al.
\newblock Mastering the game of go without human knowledge.
\newblock {\em Nature}, 550(7676):354--359, 2017.

\bibitem{tan2024beyond}
Charlie~B Tan, Edan Toledo, Benjamin Ellis, Jakob~N Foerster, and Ferenc Husz{\'a}r.
\newblock Beyond the boundaries of proximal policy optimization.
\newblock {\em arXiv preprint arXiv:2411.00666}, 2024.

\bibitem{wang2025aspo}
Jiakang Wang, Runze Liu, Lei Lin, Wenping Hu, Xiu Li, Fuzheng Zhang, Guorui Zhou, and Kun Gai.
\newblock Aspo: Asymmetric importance sampling policy optimization.
\newblock {\em arXiv preprint arXiv:2510.06062}, 2025.

\bibitem{wang2020truly}
Yuhui Wang, Hao He, and Xiaoyang Tan.
\newblock Truly proximal policy optimization.
\newblock In {\em Uncertainty in artificial intelligence}, pages 113--122. PMLR, 2020.

\bibitem{wang2019trust}
Yuhui Wang, Hao He, Xiaoyang Tan, and Yaozhong Gan.
\newblock Trust region-guided proximal policy optimization.
\newblock {\em Advances in Neural Information Processing Systems}, 32, 2019.

\bibitem{xi2025bapo}
Zhiheng Xi, Xin Guo, Yang Nan, Enyu Zhou, Junrui Shen, Wenxiang Chen, Jiaqi Liu, Jixuan Huang, Zhihao Zhang, Honglin Guo, et~al.
\newblock Bapo: Stabilizing off-policy reinforcement learning for llms via balanced policy optimization with adaptive clipping.
\newblock {\em arXiv preprint arXiv:2510.18927}, 2025.

\bibitem{xie2024simple}
Zhengpeng Xie, Qiang Zhang, and Renjing Xu.
\newblock Simple policy optimization.
\newblock {\em arXiv preprint arXiv:2401.16025}, 2024.

\bibitem{ye2020mastering}
Deheng Ye, Zhao Liu, Mingfei Sun, Bei Shi, Peilin Zhao, Hao Wu, Hongsheng Yu, Shaojie Yang, Xipeng Wu, Qingwei Guo, et~al.
\newblock Mastering complex control in moba games with deep reinforcement learning.
\newblock In {\em Proceedings of the AAAI conference on artificial intelligence}, volume~34, pages 6672--6679, 2020.

\bibitem{zuo2025ttrl}
Yuxin Zuo, Kaiyan Zhang, Li~Sheng, Shang Qu, Ganqu Cui, Xuekai Zhu, Haozhan Li, Yuchen Zhang, Xinwei Long, Ermo Hua, et~al.
\newblock Ttrl: Test-time reinforcement learning.
\newblock {\em arXiv preprint arXiv:2504.16084}, 2025.

\end{thebibliography}
\bibliographystyle{plain}

\newpage
\appendix
\onecolumn
\section{Appendix}
\subsection{Code Availability}
Below are the links to our project website and source code.

\textbf{Project website:} \url{https://bounded-ratio-rl.github.io/brrl/}.

\textbf{Code:} \url{https://github.com/bounded-ratio-rl/bounded_ratio_rl}.

\subsection{Proof of Theorem~\ref{thr:conservative}}
\label{sec:proof_policy}
\begin{proof}
Since $\forall \,s,\,\mathbb{E}_{a\sim\pi_0(\cdot|s)}[\rho]=1$, maximizing $L_{\pi_0}(\pi)$ is equivalent to maximizing $L_{\pi_0}(\pi) + \mathbb{E}_{s\sim d_{\pi_0}}[V_{\pi_0}(s)] - \eta(\pi_0)$, which can be expanded as  (from Equation~\eqref{eq:l})
\begin{align*}
    &\eta(\pi_0) + \mathbb{E}_{s\sim d_{\pi_0} +  a\sim\pi_0(\cdot|s)}[\rho A_{\pi_0}(s, a)] + \mathbb{E}_{s\sim d_{\pi_0}}[V_{\pi_0}(s)]- \eta(\pi_0)\\=& \,\mathbb{E}_{s\sim d_{\pi_0}}[V_{\pi_0}(s)\cdot\mathbb{E}_{a\sim\pi_0(\cdot|s)}[\rho]] + \mathbb{E}_{s\sim d_{\pi_0}, a\sim\pi_0(\cdot|s)}[\rho A_{\pi_0}(s, a)] \\= &\,\mathbb{E}_{s\sim d_{\pi_0}, a\sim\pi_0(\cdot|s)}[\rho V_{\pi_0}(s) + \rho A_{\pi_0}(s, a)] = \mathbb{E}_{s\sim d_{\pi_0}, a\sim\pi_0(\cdot|s)}[\rho Q_{\pi_0}(s, a)]
\end{align*}
The original problem~\eqref{eq:conservative} can then be written as
\begin{align*}
    \max_{\rho}\mathbb{E}_{\substack{s\sim d_{\pi_0}\\a\sim\pi_0(\cdot|s)}}[\rho Q_{\pi_0}(s, a) - \lambda(\rho-1+\epsilon)  \log (\rho-1+\epsilon)   -\lambda (1+\epsilon - \rho  )\log(1+\epsilon - \rho )],
\end{align*}
with the normalization constraint
\begin{align*}
    \forall\, s,\,\,\mathbb{E}_{a\sim\pi_0(\cdot|s)}[\rho] = 1.
\end{align*}
This optimization problem can be decomposed into subproblems for each state $s$.
Now, given a fixed state $s$, we solve the constrained optimization problem using the Lagrangian approach, with a Lagrangian multiplier denoted as $-\mu(s)$
\begin{align*}
   \mathcal{L}(\rho):=&\mathbb{E}_{a\sim{\pi_0(\cdot|s)}}[\rho Q_{\pi_0}(s, a) - \lambda(\rho-1+\epsilon)  \log (\rho-1+\epsilon)   -\lambda (1+\epsilon - \rho  )\log(1+\epsilon - \rho )] \\&- \mu(s) (\mathbb{E}_{a\sim\pi_0(\cdot|s)}[\rho]-1) \\
   =:& \mathbb{E}_{a\sim{\pi_0(\cdot|s)}}[f(\rho)] + \mu(s),
\end{align*}
where
\begin{align*}
    f(\rho):= \rho Q_{\pi_0}(s, a) - \lambda(\rho-1+\epsilon)  \log (\rho-1+\epsilon)   -\lambda (1+\epsilon - \rho  )\log(1+\epsilon - \rho ) - \mu(s)\rho.
\end{align*}
If $\mathcal{A}$ is continuous, we apply the calculus of variations
\begin{equation}
\begin{split}
    & \frac{\partial }{\partial\rho} f(\rho) = 0 \\
    \Rightarrow\quad &\frac{\partial }{\partial\rho}\left(\rho Q_{\pi_0}(s, a) - \lambda(\rho-1+\epsilon)  \log (\rho-1+\epsilon)   -\lambda (1+\epsilon - \rho  )\log(1+\epsilon - \rho ) - \mu(s) \rho\right)=0\\
     \Rightarrow\quad &Q_{\pi_0}(s, a) - \lambda (\log (\rho-1+\epsilon) - \log(1+\epsilon - \rho))- \mu(s) = 0 \\
     \Rightarrow\quad & \log\frac{\rho-1+\epsilon}{1+\epsilon -\rho} = \frac{Q_{\pi_0}(s, a) - \mu(s)}{\lambda} 
     \\
     \Rightarrow\quad& \rho^* = 1 + \frac{\epsilon\exp\left(\frac{Q_{\pi_0}(s, a) - \mu(s)}{\lambda}\right)-\epsilon}{1 + \exp\left(\frac{Q_{\pi_0}(s, a) - \mu(s)}{\lambda}\right)} 
     = 1 + \epsilon \tanh\left(\frac{\tilde{A}_{\pi_0}}{2\lambda}\right)\label{eq:cal_val}
\end{split}
\end{equation}
The second derivative of the objective function can be computed as $-\frac{\lambda}{\rho-1+\epsilon}-\frac{\lambda}{1+\epsilon-\rho}$, which is negative for arbitrary $1-\epsilon<\rho<1+\epsilon$.
Therefore, $\rho^*$ is the maximizer of the objective function.

The Lagrangian multiplier $\mu(s)$ should be chosen to satisfy the normalization constraint:
\begin{align*}
    \mathbb{E}_{a\sim\pi_0(\cdot|s)}\left[1 + \epsilon \tanh\left(\frac{\tilde{A}_{\pi_0}}{2\lambda}\right)\right] = 1\quad \Leftrightarrow\quad\mathbb{E}_{a\sim\pi_0(\cdot|s)}\left[\tanh\left(\frac{\tilde{A}_{\pi_0}}{2\lambda}\right)\right] = 0.\label{eq:cons}
\end{align*}
Note that such $\mu(s)$ always exists because for all $\lambda>0$, $\mathbb{E}_{a\sim\pi_0(\cdot|s)}\left[1 + \epsilon \tanh\left(\frac{\tilde{A}_{\pi_0}}{2\lambda}\right)\right]$ is a smooth function of $\mu$ with the value range between $1-\epsilon$ and $1+\epsilon$.
Now we show that the corresponding $\mu(s)$ is also the minimizer of $\mathbb{E}_{a\sim\pi_0(\cdot|s)}\left[g\left(\frac{Q_{\pi_0}(s, a)-\mu(s)}{\lambda}\right)\right]$, where $g:=\ln(e^{-\frac{x}{2}} + e^{\frac{x}{2}})$:
\begin{align*}
    &\mathbb{E}_{\pi_0}\left[\frac{\partial g}{\partial u}\left(\frac{q-u}{\lambda}\right)\right] = 0 \,\,\Leftrightarrow\,\, \mathbb{E}_{\pi_0}\left[\frac{e^{\frac{q-u}{2\lambda}} - e^{-\frac{q-u}{2\lambda}}}{2(e^{-\frac{q-u}{2\lambda}} + e^{\frac{q-u}{2\lambda}})}\right] = 0  \,\,\Leftrightarrow\,\, \mathbb{E}_{\pi_0}\left[\tanh\left(\frac{q-u}{2\lambda}\right)\right] =0,
\end{align*}
where $\mathbb{E}_{\pi_0}$ abbreviates $\mathbb{E}_{a\sim\pi_0(\cdot|s)}$.
Besides, $$\frac{\partial^2 g}{\partial u^2}=\frac{1}{4\lambda^2}\mathbb{E}_{a\sim\pi_0(\cdot|s)}\left[\frac{(e^{\frac{u-q}{2\lambda}} + e^{-\frac{u-q}{2\lambda}})^2 -(e^{\frac{u-q}{2\lambda}} - e^{-\frac{u-q}{2\lambda}})^2}{(e^{-\frac{u-q}{2\lambda}} + e^{\frac{u-q}{2\lambda}})^2}\right]\geq 0.$$
Therefore, the optimal $\mu(s)$ is the minimizer of $g$.

If $\mathcal{A}$ is discrete, for each fixed state $s$, we denote the vectorized $\rho(a)$, $Q_{\pi_0}(s, a)$, $H(\rho)$ and $\pi_0(a|s)$ as $\rho, Q, H,\pi \in \mathbb{R}^{|\mathcal{A}|}$.
Then the Lagrangian $\mathcal{L}(\rho)$ can be expressed by
\begin{align*}
    \mathcal{L}(\rho) := \pi^\top (\rho \odot Q - \lambda H - \mu \rho) + \mu,
\end{align*}
where $\odot$ denotes elementwise product.
Applying zero gradient w.r.t. $\rho$ gives
\begin{align*}
    (\text{diag}(Q) - \lambda \text{diag}(H') - \mu I)^\top \pi = \mathbf{0} \quad\Rightarrow\quad \pi(a|s) (Q_{\pi_0}(s,a) - \lambda H'(\rho) - \mu(s)) = 0, \,\,\forall \,\, a,
\end{align*}
This gives the same expression as~\eqref{eq:cal_val}, therefore, the following proof steps are the same as the continuous case.
\end{proof}

\subsection{Proof of Theorem~\ref{thr:soft_guarantee}}
\label{sec:proof_mono}

We start by proving the following Lemma:
\begin{lemma}
    Define $L_{\pi_0}^{\pi^*}(s):=\mathbb{E}_{a\sim \pi^*(\cdot|s)}[Q_{\pi_0}(s, a)]$, consider $\pi^*(a|s)=\left[1+\epsilon \tanh\left(\frac{\tilde{A}_{\pi_0}}{2\lambda}\right)\right]\pi_0(a|s)$ from Theorem~\ref{thr:conservative}, we have
    \begin{align*}
        L_{\pi_0}^{\pi^*}(s) = V_{\pi_0}(s) + \epsilon\mathbb{E}_{a\sim \pi_0(\cdot|s)}\left[\tanh\left(\frac{\Tilde{A}_{\pi_0}(s, a)}{2\lambda}\right)\Tilde{A}_{\pi_0}(s, a)\right].
    \end{align*} \label{lem:l}
\end{lemma}
\begin{proof}
    We directly compute $L_{\pi_0}^{\pi^*}(s)$ as
    \begin{align*}
        &L_{\pi_0}^{\pi^*}(s) = \mathbb{E}_{\pi_0}\left[\left(1 + \epsilon\tanh\left(\frac{\tilde{A}_{\pi_0}(s,a)}{2\lambda}\right)\right)Q_{\pi_0}(s, a)\right] \\
        &=\underbrace{\mathbb{E}_{\pi_0}[Q_{\pi_0}(s, a)]}_{V_{\pi_0}(s)} + \epsilon \mathbb{E}_{\pi_0}\left[\tanh\left(\frac{\tilde{A}_{\pi_0}(s,a)}{2\lambda}\right)Q_{\pi_0}(s,a)\right] \\
        &=V_{\pi_0}(s) + \epsilon \mathbb{E}_{\pi_0}\left[\tanh\left(\frac{\tilde{A}_{\pi_0}(s,a)}{2\lambda}\right)\cdot(\mu(s) + \tilde{A}_{\pi_0}(s, a))\right] \\
        & = V_{\pi_0}(s) + \epsilon \mathbb{E}_{\pi_0}\left[\tanh\left(\frac{\tilde{A}_{\pi_0}(s,a)}{2\lambda}\right)\tilde{A}_{\pi_0}(s, a)\right]
        + \epsilon \underbrace{\mathbb{E}_{\pi_0}\left[\tanh\left(\frac{\tilde{A}_{\pi_0}(s,a)}{2\lambda}\right)\right]}_{=0}\mu(s).
    \end{align*}
    where we abbreviate $\mathbb{E}_{a\sim \pi_0(\cdot|s)}$ with $\mathbb{E}_{\pi_0}$.
    We have $\mathbb{E}_{\pi_0}\left[\tanh\left(\frac{\tilde{A}_{\pi_0}(s,a)}{2\lambda}\right)\right]=0$ because of the normalization constraint for $\mu$ in Theorem~\ref{thr:conservative}.
\end{proof}

We now prove Theorem~\ref{thr:soft_guarantee} using matrix representations for MDP with a discrete state and action space. The results for continuous spaces can be extended by using linear operators other than matrices.
\begin{proof}(Theorem~\ref{thr:soft_guarantee})
   We define $r_{\pi}\in \mathbb{R}^{|\mathcal{S}|}$ with $r_{\pi}(s):=\mathbb{E}_{a\sim \pi(\cdot|s),s'\sim P(s'|s, a)}[r(s, a, s')]$.
    The transition kernel $P_\pi \in \mathbb{R}^{|\mathcal{S}|\times|\mathcal{S}|}$ is defined with $P_\pi(s, s'):=\sum_{a\sim \pi(\cdot|s)} \pi(a|s)p(s'|s, a)$.
    We denote $(I-\gamma P_{\pi})^{-1}:=(I + \gamma P_{\pi} +\gamma^2 P_\pi^2 + ...)$.
Then, given the initial distribution denoted as $d\in [0, 1]^{|\mathcal{S}|},d(s):=d_0(s)$, we can express the state visitation distribution $d_\pi \in [0, \frac{1}{1-\gamma}]^{|\mathcal{S}|}$ as $d_\pi^\top = d^\top(I - \gamma P_\pi)^{-1}$.

     Let us define $V_{\pi}\in \mathbb{R}^{|\mathcal{S}|}$, where each component $s$ corresponds to $V_{\pi}(s)$. 
   The definition of the value function implies that
   \begin{equation}
   \begin{split}
       & V_{\pi}= r_\pi + \gamma P_\pi V_\pi \\
       &V_{\pi}=(I - \gamma P_\pi)^{-1}r_\pi. \label{eq:vec_value}
   \end{split}
   \end{equation}
   We also have
   \begin{equation}
       \eta(\pi) = \mathbb{E}_{s\sim d_\pi, a\sim \pi(\cdot|s), s'\sim P(\cdot|s, a)}[r(s,a,s')] = d^\top(I - \gamma P_\pi)^{-1} r_\pi = d^\top V_\pi \label{eq:vec_eta}
   \end{equation}
    
    We then define $L^{\pi_2}_{\pi_1}\in \mathbb{R}^{|\mathcal{S}|}$ with 
    \begin{equation}
   L^{\pi_2}_{\pi_1} = r_{\pi_2} + \gamma P_{\pi_2} V_{\pi_1},\label{eq:vec_surr}
   \end{equation}
    which aligns with the definition of $L_{\pi_0}^{\pi^*}(s)$ in Lemma~\ref{lem:l}.
    Let us denote $B_1 \in \mathbb{R}^{|\mathcal{S}|}$ with $B_1(s) := \mathbb{E}_{a\sim\pi_0(\cdot|s)}\left[\tanh\left(\frac{\tilde{A}_{\pi_0}(s,a)}{2\lambda}\right)\tilde{A}_{\pi_0}(s, a)\right]$.
    Then the state-wise result Lemma~\ref{lem:l} can be rewritten for the full state space as 
    \begin{equation}
   L_{\pi_0}^{\pi^*} = V_{\pi_0} + \epsilon B_1, \label{eq:vec_improve}
   \end{equation} 
   where $B_1$ is positive along each of its components.
    Combining~\eqref{eq:vec_surr} and~\eqref{eq:vec_improve} gives
    \begin{equation}
        r_{\pi^*} + \gamma P_{\pi^*} V_{\pi_0} =: L_{\pi_0}^{\pi^*} =  V_{\pi_0} + \epsilon B_1 \,\,\Rightarrow\,\,r_{\pi^*} = (I - \gamma P_{\pi^*})V_{\pi_0} + \epsilon B_1 \label{eq:vec_r}
    \end{equation}
    On the other hand, applying~\eqref{eq:vec_value} to $\pi^*$ in combination with~\eqref{eq:vec_r} gives
    \begin{equation}
       V_{\pi^*} = (I - \gamma P_{\pi^*})^{-1} r_{\pi^*} =(I - \gamma P_{\pi^*})^{-1}((I - \gamma P_{\pi^*})V_{\pi_0} + \epsilon B_1) = V_{\pi_0} + \epsilon (I - \gamma P_{\pi^*})^{-1} B_1.
    \end{equation}
    Finally, we can obtain from~\eqref{eq:vec_eta}
    \begin{align*}
        \eta(\pi^*) := d^\top V_{\pi^*} = d^\top V_{\pi_0} + \epsilon d^\top(I - \gamma P_{\pi^*})^{-1} B_1 = \eta(\pi_0) + \epsilon \mathbb{E}_{s\sim d_{\pi^*}}[B_1(s)]
    \end{align*}
    Applying the definition of $B_1(s)$ finishes the proof.
\end{proof}

\subsection{Proof of Corollary~\ref{coro:soft_guarantee_true_l}}
\label{sec:proof_mono_true_l}

We start by proving the following Lemma on per-state performance improvement.
    \begin{lemma}
    Define $L_{\pi_0}^{\pi_\theta}(s):=\mathbb{E}_{a\sim \pi_\theta(\cdot|s)}[Q_{\pi_0}(s, a)]$, $D_{\theta}^{ATV}(s):=\sum_a |(\pi_\theta(a|s)-\pi^*(a|s))A_{\pi_0}(s, a)|$, we have
    \begin{align*}
        L_{\pi_0}^{\pi_\theta}(s) \geq V_{\pi_0}(s) + \epsilon\mathbb{E}_{a\sim \pi_0(\cdot|s)}\left[\tanh\left(\frac{\Tilde{A}_{\pi_0}(s, a)}{2\lambda}\right)\Tilde{A}_{\pi_0}(s, a)\right] - D_{\theta}^{ATV}(s)
    \end{align*} \label{lem:l_theta}
\end{lemma}
\begin{proof}
    We first bound $|L_{\pi_0}^{\pi^*}(s) - L_{\pi_0}^{\pi_\theta}(s)|$ with $D_{\theta}^{ATV}(s)$
    \begin{align*}
        |L_{\pi_0}^{\pi^*}(s) - L_{\pi_0}^{\pi_\theta}(s)| = &|\mathbb{E}_{a\sim \pi^*(\cdot|s)}\left[Q_{\pi_0}(s,a)\right] -\mathbb{E}_{a\sim \pi_\theta(\cdot|s)}\left[Q_{\pi_0}(s,a)\right]| \\
        = &|\mathbb{E}_{a\sim \pi^*(\cdot|s)}\left[Q_{\pi_0}(s,a)\right] - V_{\pi_0}(s) -(\mathbb{E}_{a\sim \pi_\theta(\cdot|s)}\left[Q_{\pi_0}(s,a)\right] - V_{\pi_0}(s))| \\
        = &|\mathbb{E}_{a\sim \pi^*(\cdot|s)}\left[A_{\pi_0}(s,a)\right] -\mathbb{E}_{a\sim \pi_\theta(\cdot|s)}\left[A_{\pi_0}(s,a)\right]| \\
        =&\left|\mathbb{E}_{a\sim \pi_0(\cdot|s)}\left[\left(\frac{\pi^*(a|s)}{\pi_0(a|s)}-\frac{\pi_\theta(a|s)}{\pi_0(a|s)}\right)A_{\pi_0}(s,a)\right]\right|\\
        \leq &\mathbb{E}_{a\sim \pi_0(\cdot|s)}\left[\left|\left(\frac{\pi^*(a|s)}{\pi_0(a|s)}-\frac{\pi_\theta(a|s)}{\pi_0(a|s)}\right)A_{\pi_0}(s,a)\right|\right]\\
        =&D_{\theta}^{ATV}(s)
    \end{align*}
    Then we have from Lemma~\ref{lem:l}
    \begin{align*}
        L_{\pi_0}^{\pi_\theta}(s) \geq L_{\pi_0}^{\pi^*}(s) - D_{\theta}^P(s) = V_{\pi_0}(s) + \epsilon\mathbb{E}_{a\sim \pi_0(\cdot|s)}\left[\tanh\left(\frac{\Tilde{A}_{\pi_0}(s, a)}{2\lambda}\right)\Tilde{A}_{\pi_0}(s, a)\right] - D_{\theta}^{ATV}(s)
    \end{align*}
\end{proof}

We now prove the Corollary~\ref{coro:soft_guarantee_true_l}.
\begin{proof}
    Similar to the proof of~\ref{thr:soft_guarantee}, we start by defining $D_\theta^{ATV}\in \mathbb{R}^{|\mathcal{S}|}$ with each component as $D_\theta^{ATV}(s)$.
    Then we can express Lemma~\ref{lem:l_theta} in the full state space as 
    $L_{\pi_0}^{\pi_\theta} \succeq V_{\pi_0} + \epsilon B_1 - D_\theta^{ATV}$, where $\succeq$ denotes elementwise $\geq$.
    Following similar steps of the proof of Theorem~\ref{thr:soft_guarantee}, we obtain 
    \begin{equation}
        r_{\pi_\theta} \succeq (I - \gamma P_{\pi_\theta})V_{\pi_0} + \epsilon B_1 - D_\theta^{ATV}. \label{ineq:vec_r}
    \end{equation}
    Multiplying both sides by $(I - \gamma P_{\pi_\theta})^{-1}$ from left gives
    \begin{equation}
        V_{\pi_\theta} \succeq V_{\pi_0} + (I - \gamma P_{\pi_\theta})^{-1}(\epsilon B_1 - D_\theta^{ATV}), \label{eq:vec_bound}
    \end{equation}
    because all terms of $(I - \gamma P_{\pi_\theta})^{-1}$ are positive. 
    Multiplying both sides by $d^\top$ finishes the proof.
\end{proof}

\subsection{Proof of Corollary~\ref{coro:loss_tv}}
\label{sec:proof_coro_loss_tv}
\begin{proof}
    From the inequality~\eqref{eq:vec_bound}, we can obtain 
    \begin{equation}
\begin{split}
    V_{\pi_\theta} \succeq &V_{\pi_0} + (I - \gamma P_{\pi_\theta})^{-1}(\epsilon B_1 - D_\theta^{ATV}) =V_{\pi_0} + \epsilon(I - \gamma P_{\pi_\theta})^{-1}B_1 - (I - \gamma P_{\pi_\theta})^{-1} D_\theta^{ATV} \\
    =&V_{\pi_0} + \epsilon(I - \gamma P_{\pi_\theta})^{-1}B_1 - (I - \gamma P_{\pi_0})^{-1} D_\theta^{ATV} - ((I - \gamma P_{\pi_\theta})^{-1} - (I - \gamma P_{\pi_0})^{-1})D_\theta^{ATV} \\
    =&V_{\pi_0} + \epsilon(I - \gamma P_{\pi_\theta})^{-1}B_1 - (I - \gamma P_{\pi_0})^{-1} D_\theta^{ATV} \\-& \underbrace{\gamma(I - \gamma P_{\pi_0})^{-1}\underbrace{(P_{\pi_\theta} - P_{\pi_0})\underbrace{(I - \gamma P_{\pi_\theta})^{-1}D_\theta^{ATV}}_{:=X}}_{:=Y} }_{:=Z}
\end{split} \label{ineq:longest}
\end{equation}
We can first bound $X$ elementwise by 
\begin{equation*}
    X \preceq (I - \gamma P_{\pi_\theta})^{-1} D^{ATV}_{\max} = \mathbf{1} \cdot \frac{D^{ATV}_{\max}}{1-\gamma},
\end{equation*}
where $\preceq$ denotes elementwise smaller or equal to.
Then we can bound each term $Y$ by
\begin{align*}
    |Y(s)| = &|\sum_a (\pi_\theta(a|s) - \pi_0(a|s))\sum_{s'}p(s'|s, a) X(s')| \\
    \leq &\sum_a |\pi_\theta(a|s) - \pi_0(a|s)|\sum_{s'}p(s'|s, a) \max_{s'}|X(s')| \\
    =&\sum_a |\pi_\theta(a|s) - \pi_0(a|s)| \max_{s'}|X(s')| \\
    \leq & (\underbrace{\sum_a |\pi_0(a|s) - \pi^*(a|s)|}_{\mathbb{E}_{a\sim \pi_0(\cdot|s)}[|\frac{\pi^*(a|s)}{\pi_0(a|s)}-1|]} + \sum_a |\pi_\theta(a|s) - \pi^*(a|s)|) \max_{s'}|X(s')|\\
    \leq &(\epsilon + D^{TV}_\theta(s) )\cdot \frac{D^{ATV}_{\max}}{1-\gamma} 
\end{align*}
Finally, $Z$ satisfies
\begin{equation}
    |Z| =  |\gamma (I-\gamma P_{\pi_0})^{-1}Y| \preceq \gamma(I-\gamma P_{\pi_0})^{-1}(D_\theta^{TV} + \epsilon \cdot \mathbf{1}) \cdot \frac{D^{ATV}_{\max}}{1-\gamma} \label{ineq:z}
\end{equation}
And 
\begin{equation}
\begin{split}
    |d^\top Z| & \leq  \gamma d^\top(I-\gamma P_{\pi_0})^{-1} (D_\theta^{TV} + \epsilon \cdot \mathbf{1}) \cdot \frac{D^{ATV}_{\max}}{1-\gamma} \\
    & = \frac{\gamma D^{ATV}_{\max}}{1-\gamma}\mathbb{E}_{s\sim d_{\pi_0}}[D_\theta^{TV}(s)] + \frac{\gamma \epsilon D^{ATV}_{\max}}{(1-\gamma)^2} \label{ineq:dz}
\end{split}
\end{equation}
Now let us consider the second term of~\eqref{ineq:longest} as
\begin{align*}
    \epsilon(I - \gamma P_{\pi_\theta})^{-1}B_1 =  \epsilon(I - \gamma P_{\pi^*})^{-1}B_1 - \epsilon ((I - \gamma P_{\pi^*})^{-1}-(I - \gamma P_{\pi_\theta})^{-1})B_1.
\end{align*}
With similar techniques for bounding $X$, $Y$ and $Z$ from~\eqref{ineq:longest} to~\eqref{ineq:z}, we have
\begin{align*}
    &|d^\top ((I - \gamma P_{\pi^*})^{-1}-(I - \gamma P_{\pi_\theta})^{-1})B_1| \leq \gamma d^\top (I - \gamma P_{\pi^*})^{-1} D_\theta^{TV} \cdot \frac{\Tilde{\delta}}{1-\gamma} \\
    &\leq \gamma d^\top (I - \gamma P_{\pi^*})^{-1} \mathbf{1}\cdot D_{\max}^{TV} \cdot \frac{\Tilde{\delta}}{1-\gamma} = \frac{\gamma \Tilde{\delta} D_{\max}^{TV}}{(1-\gamma)^2}
\end{align*}
Combining this with~\eqref{ineq:dz} and~\eqref{ineq:longest} finishes the proof.
\end{proof}
The results for continuous spaces can be extended by using linear operators other than matrices.

\subsection{Proof of Proposition~\ref{prop:ppo_equ}}
\label{app:proof_ppo_equ}
We use $A_0$ to abbreviate $A_{\pi_0}$. When $A_0>0$, the negative PPO objective for a fixed state-action pair can be further expressed as
\begin{align*}
&l^{PPO}(\rho) =
\begin{cases}
-\rho A_0, & \rho \le 1+\epsilon \\[4pt]
-(1+\epsilon) A_0, & \rho > 1+\epsilon
\end{cases}  
\\\quad\Rightarrow\quad &l'(\rho) :=l^{PPO}(\rho) + (1+\epsilon)A_0=
\begin{cases}
[(1+\epsilon) - \rho] A_0, & \rho \le 1+\epsilon \\[4pt]
0, & \rho > 1+\epsilon
\end{cases}\quad,
\end{align*}
where $l'$ is constructed from adding a constant $(1+\epsilon)A_0$ to $l^{PPO}$.
Similarly, we construct $l'$ for $A_0\leq0$:
\begin{align*}
&l^{PPO}(\rho) =
\begin{cases}
-\rho A_0, & \rho \ge 1-\epsilon \\[4pt]
-(1-\epsilon) A_0, & \rho < 1-\epsilon
\end{cases}  
\quad \\\Rightarrow\quad  &l'(\rho) := l^{PPO}(\rho) + (1-\epsilon)A_0=
\begin{cases}
[(1-\epsilon) - \rho] A_0, & \rho \ge 1-\epsilon \\[4pt]
0, & \rho < 1-\epsilon
\end{cases}\quad,
\end{align*}
We can rearrange $l'$ as
\begin{equation*}
\begin{split}
        &l'(\rho ):=\begin{cases}|A_{0}|\cdot |\rho - (1+\epsilon\cdot\text{sign}(A_{0}) )|, & [\rho - (1+\epsilon\cdot\text{sign}(A_{0}))]\cdot A_{0} \leq 0, \\[2pt]
0, & \text{Otherwise}.
\end{cases}
\end{split}
\end{equation*}
Then we have
\begin{align*}
    \mathbb{E}_{s\sim d_{\pi_0},a\sim \pi_0(\cdot|s)}[l'(\rho)] &= \mathbb{E}_{s\sim d_{\pi_0},a\sim \pi_0(\cdot|s)}[l^{PPO}(\rho)] \\ 
    &+ \mathbb{E}_{s\sim d_{\pi_0},a\sim \pi_0(\cdot|s)}[(1+\epsilon)\mathbb{I}(A_0>0)A_0+(1-\epsilon)\mathbb{I}(A_0\leq0)A_0],
\end{align*}
where the second term is a constant independent of $\rho$.
Therefore, $l'$ is an equivalent loss function to $l^{PPO}$ for solving the optimal $\rho$.

\subsection{Proof of Corollary~\ref{coro:unbalance}}
Similar to Section~\ref{sec:proof_policy}, for each fixed $s$, the Lagrangian of problem~\eqref{prob:unbalanced} can be written as
\begin{align*}
    \mathcal{L}(\rho)=&\mathbb{E}_{a\sim{\pi_0(\cdot|s)}}\left[\rho Q_{\pi_0}(s, a) - \lambda(\rho-c_l)  \log (\rho-c_l)   -\lambda (c_h - \rho  )\log(c_h - \rho ) - \lambda \log \frac{c_h-1}{1-c_l} \rho\right] \\- &\mu(s) (\mathbb{E}_{a\sim\pi_0(\cdot|s)}[\rho]-1)
\end{align*}
For continuous MDP, applying $\frac{\partial L}{\partial\rho}=0$ gives
\begin{align*}
    &Q_{\pi_0}(s, a) - \lambda \log (\rho - c_l) + \lambda \log (c_h - \rho) - \lambda \log \frac{c_h-1}{1-c_l} - \mu(s) = 0 \\
    \Leftrightarrow \quad & Q_{\pi_0}(s, a) - \mu(s) = \lambda \left(\log \frac{\rho - c_l}{c_h - \rho} + \log \frac{c_h-1}{1-c_l}\right) \\
    \Leftrightarrow \quad & \frac{\rho - c_l}{c_h - \rho} = \frac{1 - c_l}{c_h - 1}\exp{\frac{Q_{\pi_0}(s, a) - \mu(s)}{\lambda}} \\
    \Leftrightarrow \quad & \rho = \frac{c_l + c_h\cdot\frac{1 - c_l}{c_h - 1}\exp{\frac{Q_{\pi_0}(s, a) - \mu(s)}{\lambda}} }{1 + \frac{1 - c_l}{c_h - 1}\exp{\frac{Q_{\pi_0}(s, a) - \mu(s)}{\lambda}}} = c_l + \frac{c_h - c_l}{1 + \frac{c_h - 1}{1 - c_l}\exp{\frac{\mu(s) - Q_{\pi_0}(s, a)}{\lambda}}}
\end{align*}
Now we derive the conditions that need to be satisfied by $\mu'_{\pi_0}(s)$.
Specifically, having $\rho$ normalized gives
\begin{align*}
    &\mathbb{E}_{a\sim \pi_0(\cdot|s)}\left[c_l + \frac{c_h - c_l}{1 + \frac{c_h - 1}{1 - c_l}\exp{\frac{\mu(s) - Q_{\pi_0}(s, a)}{\lambda}}}\right] = 1 \\\,\,\Leftrightarrow \,\,&\mathbb{E}_{a\sim \pi_0(\cdot|s)}\left[\frac{c_h - c_l}{1 + \frac{c_h - 1}{1 - c_l}\exp{\frac{\mu(s) - Q_{\pi_0}(s, a)}{\lambda}}}\right] = 1-c_l \\
    \,\,\Leftrightarrow \,\,& \mathbb{E}_{a\sim \pi_0(\cdot|s)}\left[\frac{c_h - c_l - (1-c_l) - (c_h-1)\exp{\frac{\mu(s) - Q_{\pi_0}(s, a)}{\lambda}}}{1 + \frac{c_h - 1}{1 - c_l}\exp{\frac{\mu(s) - Q_{\pi_0}(s, a)}{\lambda}}}\right] = 0 \\
    \,\,\Leftrightarrow \,\,& \mathbb{E}_{a\sim \pi_0(\cdot|s)}\left[\frac{1 - \exp{\frac{\mu(s) - Q_{\pi_0}(s, a)}{\lambda}}}{1 + \frac{c_h - 1}{1 - c_l}\exp{\frac{\mu(s) - Q_{\pi_0}(s, a)}{\lambda}}}\right] = 0
\end{align*}
On the other hand, $\frac{\partial g'}{\partial x}=0$ from $g'$ defined in Corollary~\ref{coro:unbalance} gives
\begin{align*}
   &\mathbb{E}_{a\sim\pi_0(\cdot|s)} \left[\frac{\frac{c_h-1}{c_h-c_l}e^{\frac{c_h-1}{c_h-c_l}x} - \frac{c_h-1}{1-c_l}\cdot\frac{1-c_l}{c_h-c_l}e^{-\frac{1-c_l}{c_h-c_l}x}}{e^{\frac{c_h-1}{c_h-c_l}x} + \frac{c_h-1}{1-c_l}e^{-\frac{1-c_l}{c_h-c_l}x}}\right] = 0 \\
   \Leftrightarrow \quad &\mathbb{E}_{a\sim\pi_0(\cdot|s)} \left[\frac{e^{\frac{c_h-1}{c_h-c_l}x} - e^{-\frac{1-c_l}{c_h-c_l}x}}{e^{\frac{c_h-1}{c_h-c_l}x} + \frac{c_h-1}{1-c_l}e^{-\frac{1-c_l}{c_h-c_l}x}}\right] = 0 \\
   \Leftrightarrow \quad &\mathbb{E}_{a\sim\pi_0(\cdot|s)} \left[\frac{1 - e^{-x}}{1 + \frac{c_h-1}{1-c_l}e^{-x}}\right] = 0,
\end{align*}
where assigning $x=\frac{\tilde{A}'_{\pi_0}}{\lambda}$ recovers the normalization condition above.
For a discrete MDP, similar to the proof of Theorem~\ref{thr:soft_guarantee}, we can obtain the same Lagrangian as the continuous case.

\subsection{Monotonic Guarantees for Asymmetric Bounded Ratio RL}
\label{app:proof_asy_guarantee}
\begin{corollary}[Asymmetric monotonic performance guarantee]
    The optimal policies in Theorem~\ref{coro:unbalance} satisfy
    \begin{align*}
        \eta(\pi^*) =\,& \,\eta(\pi_0) + (c_h - 1) \mathbb{E}_{s\sim d_{\pi^*}, a\sim \pi_0(\cdot|s)}\left[\frac{1 + e^{-\tilde{A}'_{\pi_0}/{\lambda}}}{1 + \frac{c_h-1}{1-c_l}e^{-\tilde{A}'_{\pi_0}/{\lambda}}}\cdot \tanh\left(\frac{\tilde{A}'_{\pi_0}}{2\lambda}\right)\tilde{A}'_{\pi_0}\right]\\
        =:\,&\eta(\pi_0) + (c_h - 1) B',
    \end{align*}
    where $\tilde{A}_{\pi_0}$ abbreviates $\tilde{A}_{\pi_0}(s,a)$, $B'$ is a non-negative constant given fixed $\pi_0$.
     \label{coro:unbalance_soft_guarantee}
\end{corollary}

\begin{proof}
    We start by deriving $L^{\pi^*}_{\pi_0}(s)$, similar to Lemma~\ref{lem:l}.
    \begin{equation}
    \begin{split}
        L^{\pi^*}_{\pi_0}(s) = &\,\mathbb{E}_{a\sim \pi_0(\cdot|s)}\left[\frac{\pi^*(a|s)}{\pi_0(a|s)} Q_{\pi_0}(s, a)\right] \\
        = & \,V_{\pi_0}(s) + \mathbb{E}_{a\sim \pi_0(\cdot|s)}\left[\left(\frac{\pi^*(a|s)}{\pi_0(a|s)} -1\right)Q_{\pi_0}(s, a)\right] \\
        = & \,V_{\pi_0}(s) + \mu_{\pi_0}(s)\cdot\underbrace{\mathbb{E}_{a\sim \pi_0(\cdot|s)}\left[\left(\frac{\pi^*(a|s)}{\pi_0(a|s)} -1\right)\right]}_{=0}\\
        + &\,\mathbb{E}_{a\sim \pi_0(\cdot|s)}\left[\left(\frac{\pi^*(a|s)}{\pi_0(a|s)} -1\right)\tilde{A}'_{\pi_0}(s, a)\right]
    \end{split}\label{eq:s_unbalance}
    \end{equation}
where $\mathbb{E}_{a\sim \pi_0(\cdot|s)}\left[\left(\frac{\pi^*(a|s)}{\pi_0(a|s)} -1\right)\right]=0$ is because of the normalization constraints enforced by the definition of $\mu'_{\pi_0}(s)$.
We now compute $\frac{\pi^*(a|s)}{\pi_0(a|s)} -1$ from Corollary~\ref{coro:unbalance}:
\begin{align*}
    \frac{\pi^*(a|s)}{\pi_0(a|s)} -1 =& c_l + \frac{c_h - c_l}{1 + \frac{c_h-1}{1-c_l}\exp(-\tilde{A}'_{\pi_0}/{\lambda})} - 1 \\
    = & \frac{c_h - c_l + c_l-1 - (c_h-1)\exp(-\tilde{A}'_{\pi_0}/{\lambda})}{1 + \frac{c_h-1}{1-c_l}\exp(-\tilde{A}'_{\pi_0}/{\lambda})} \\
    = & (c_h - 1)\cdot \frac{1 - \exp(-\tilde{A}'_{\pi_0}/{\lambda})}{1 + \frac{c_h-1}{1-c_l}\exp(-\tilde{A}'_{\pi_0}/{\lambda})} \\
    = & (c_h - 1)\cdot \frac{1 + \exp(-\tilde{A}'_{\pi_0}/{\lambda})}{1 + \frac{c_h-1}{1-c_l}\exp(-\tilde{A}'_{\pi_0}/{\lambda})} \cdot \frac{1 - \exp(-\tilde{A}'_{\pi_0}/{\lambda})}{1 + \exp(-\tilde{A}'_{\pi_0}/{\lambda})} \\
    = & (c_h - 1)\cdot \frac{1 + \exp(-\tilde{A}'_{\pi_0}/{\lambda})}{1 + \frac{c_h-1}{1-c_l}\exp(-\tilde{A}'_{\pi_0}/{\lambda})} \cdot \tanh\left(\frac{\tilde{A}'_{\pi_0}}{2\lambda}\right)
\end{align*}
Apply this into~\eqref{eq:s_unbalance} gives
\begin{align*}
    L^{\pi^*}_{\pi_0}(s) =  V_{\pi_0}(s) + (c_h - 1)\cdot \mathbb{E}_{a\sim \pi_0(\cdot|s)}\left[\frac{1 + \exp(-\tilde{A}'_{\pi_0}/{\lambda})}{1 + \frac{c_h-1}{1-c_l}\exp(-\tilde{A}'_{\pi_0}/{\lambda})} \cdot \tanh\left(\frac{\tilde{A}'_{\pi_0}}{2\lambda}\right)\tilde{A}'_{\pi_0}\right]
\end{align*}

Following the same steps of proof of Theorem~\ref{thr:soft_guarantee} in Appendix~\ref{sec:proof_mono}, we can apply expectations over states and finish the proof.
Notably, since $\frac{1 + e^{-\tilde{A}'_{\pi_0}/{\lambda}}}{1 + \frac{c_h-1}{1-c_l}e^{-\tilde{A}'_{\pi_0}/{\lambda}}}$ is always positive, $\frac{\pi^*(a|s)}{\pi_0(a|s)} -1$ still has the same sign as $\tilde{A}'_{\pi_0}$, which makes $B'$ non-negative.
\end{proof}

With Corollary~\ref{coro:unbalance_soft_guarantee}, one can also derive corresponding corollaries for asymmetric BPO like Corollary~\ref{coro:soft_guarantee_true_l} and~\ref{coro:loss_tv}.

\subsection{Hyperparameters}
\label{sec:hyper}
The tuned hyperparameters for BPO and PPO in RL environments are summarized in Table~\ref{tab:hyper} and Table~\ref{tab:hyper_ppo}.
For the GBPO implementation, we set the discount factor $\gamma = 1$ and utilize a mini-batch size of 1. 
Other hyperparameters are the same as those in the original training scripts provided by TTRL~\cite{zuo2025ttrl}. All the LLM finetuning experiments are conducted on 4 x NVIDIA H100 GPUs.
We set the group size to 32 and the maximum sequence length to 4,096 tokens.

\begin{table}[!hbtp]
\caption{Hyperparameters of BPO for benchmarking environments.
We take $\lambda=0.001$, $\alpha_1=0$ and $w_1=w_2=0.5$ across all environments.
The GAE-$\lambda$ is set to 0.95 for all environments except 0.98 for Swimmer.
ADP abbreviates adaptive learning rates based on the KL-divergence, according to~\cite{schwarke2025rsl}.
}
\footnotesize
\centering
\begin{tabular}{ccccccccc}
\hline
Envs           & batch size           & clip & ent\_coef & gamma                & lr        & n\_epochs & n\_steps             & n\_envs              \\ \hline
Atari          & 256                  & 0.3        & 0.001              & 0.98                 & $2.5e^{-4}$              & 5         & 128                  & 8                    \\
Ant-v4         & 256                  & 0.3        & 0                  & 0.99                 & $1e^{-4}$               & 10        & 2048                 & 1                    \\
Humanoid-v4    & 128                  & 0.2        & {0}            & 0.99                 & $1e^{-4}$               & 5         & 512                  & 1                    \\
Hopper-v4      & 32                   & 0.25       & 0                  & 0.999                & $9.808e^{-5}$           & 10        & 512                  & 4                    \\
Swimmer-v4     & 256                  & 0.1        & 0                  & 0.9999               & $3e^{-4}$               & 10        & 1024                 & 4                    \\
Go1-Rough & 24576                & 0.3        & 0.001                  & 0.99                 & ADP          & 10         & 24                   & 4096                 \\
Anymal-C & 24576                & 0.25        & 0                  & 0.99                 & ADP          & 5         & 24                   & 4096                 \\
G1-Rough       & 24576                & 0.2        & 0                  & 0.99                 & ADP          & 10        & 24                   & 4096                 \\
H1-Rough       & 24576                & 0.2        & 0                  & 0.99                 & ADP          & 5         & 24                   & 4096                 \\ \hline
\end{tabular}
\label{tab:hyper}
\end{table}

\begin{table}[!hbtp]
\caption{Hyperparameters of PPO for benchmarking environments, based on RL-Zoo~\cite{rl-zoo3}.
The GAE-$\lambda$ is set to 0.95 for all environments except 0.98 for Swimmer, 0.99 for Hopper, 0.9 for Humanoid, and 0.8 for Ant.
ADP abbreviates adaptive learning rates based on the KL-divergence, according to~\cite{schwarke2025rsl}.
}
\footnotesize
\centering
\begin{tabular}{ccccccccc}
\hline
Envs           & batch size           & clip & ent\_coef & gamma                & lr        & n\_epochs & n\_steps             & n\_envs              \\ \hline
Atari          & 256                  & 0.2        & 0.01              & 0.98                 & $2.5e^{-4}$              & 4         & 128                  & 8                    \\
Ant-v4         & 32                  & 0.1        & $4.96e^{-7}$                  & 0.98                 & $1.9e^{-5}$               & 10        & 512                 & 1                    \\
Humanoid-v4    & 256                  & 0.3        & 0.00238            & 0.98                 & $3.57e^{-5}$               & 5         & 512                  & 1                    \\
Hopper-v4      & 32                   & 0.25       & 0                  & 0.999                & $9.808e^{-5}$           & 10        & 512                  & 4                    \\
Swimmer-v4     & 256                  & 0.1        & 0                  & 0.9999               & $6e^{-4}$               & 10        & 1024                 & 4                    \\
Go1-Rough & 24576                & 0.2        & 0.01                  & 0.99                 & ADP          & 5         & 24                   & 4096                 \\
Anymal-C & 24576                & 0.2        & 0.005                  & 0.99                 & ADP          & 5         & 24                   & 4096                 \\
G1-Rough       & 24576                & 0.2        & 0.008                  & 0.99                 & ADP          & 5        & 24                   & 4096                 \\
H1-Rough       & 24576                & 0.2        & 0.01                 & 0.99                 & ADP          & 5         & 24                   & 4096                 \\ \hline
\end{tabular}
\label{tab:hyper_ppo}
\end{table}

\textbf{Empirical Tuning Observations:} 
While BPO can generally be initialized using hyperparameters tuned for PPO, specific adjustments often yield superior performance. 
Empirically, we found that increasing the clip ratio by $0.1$ and doubling the number of epochs (e.g., from $5$ to $10$) can sometimes enhance stability and results. 
Furthermore, BPO exhibits a higher sensitivity to the entropy coefficient; in many environments, an entropy weight $10^{-1}$ smaller than the optimal PPO setting is sufficient to maintain adequate exploration without destabilizing the policy.

\end{document}